%% file: main.tex
\documentclass[11pt]{article}
\input{preamble}

\begin{document}

\title{Efficient Learning of Mixed Membership Models}
\author{Zilong Tan\\
Department of Computer Science\\
Duke University \\
Email: {\tt ztan@cs.duke.edu}
\and
Sayan Mukherjee\\
Departments of Statistical Science\\
Computer Science, Mathematics, \\
Biostatistics \& Bioinformatics \\
Duke University \\
Email: {\tt sayan@stat.duke.edu}}

\maketitle
%%%%%%%%%%%%%%%%%%%%%%%%%%%%%%%%%%%%%%%%%%%%%%%%%%%%%%%%%%%%%%%%%%%%%%%%%%%%%%

\begin{abstract}
We present an efficient algorithm for learning mixed membership models when the number of variables $p$ is much larger than the number of hidden components $k$. This algorithm reduces the computational complexity of state-of-the-art tensor methods, which require decomposing an $O\left(p^3\right)$ tensor, to factorizing $O\left(p/k\right)$ sub-tensors each of size $O\left(k^3\right)$. In addition, we address the issue of negative entries in the empirical method of moments based estimators. We provide sufficient conditions under which our approach has provable guarantees. Our approach obtains competitive empirical results on both simulated and real data.
\end{abstract}

\section{Introduction}

Mixed membership models \cite{Woodbury,pritchard_inference_2000,pritchard_association_2000,blei_latent_2003,Erosheva2005} have been used extensively across applications ranging from modeling population structure in genetics  \cite{pritchard_inference_2000,pritchard_association_2000} to topic modeling of documents \cite{Woodbury,blei_latent_2003,Erosheva2005}. Mixed membership models use  Dirichlet latent variables to define cluster membership where samples can partially belong to each of $k$ latent components. Parameter estimation for such latent variables models (LVMs) using maximum likelihood methods such as expectation maximization is computationally intensive for large data, for example, if number of samples $n$ is large.

Parameter estimation using the method of moments for LVMs is an attractive scalable alternative that has been shown to have certain theoretical and computational advantages over maximum likelihood methods in the setting when $n$ is large. For LVMs, method of moments approaches reduce to tensor methods---the moments of the model parameters are expressed as a function of statistics of the observations in a tensor form. Inference in this setting becomes a problem of tensor factorization. Computational advantages of using tensor methods have been observed for many popular models, including latent Dirichlet allocation \cite{Anandkumar12}, spherical Gaussian mixture models \cite{Hsu13}, hidden Markov models \cite{Anandkumar12}, independent component analysis \cite{Comon10}, and multi-view models \cite{Anandkumar14}. An appealing property of tensor methods is the guarantee of a unique decomposition under mild conditions \cite{Kruskal77,Leurgans93}. 

There are two complications to using standard tensor decomposition methods \cite{Anandkumar16,Anandkumar14,Gu14,Kuleshov15,Colombo16,Kim14,Chi12} for LVMs. The first problem is computation and space complexity. Given $p$ variables in the LVM, parameter inference requires factorizing typically a non-orthogonal estimator tensor of size $O\left(p^3\right)$ \cite{Anandkumar14,Kuleshov15,Colombo16}, which is prohibitive for large $p$. When the estimator is orthogonal and symmetric, this can be done in $O\left(p^2 \log p\right)$ \cite{Wang15}. Online tensor decomposition \cite{Huang13} uses dimension reduction to instead factorize a reduced $k$-by-$k$-by-$k$ tensor. However, the dimension reduction can be slower than decomposing the estimator directly for large sample sizes, as well as suffer from high variance \cite{Wang15}. We introduce a simple factorization with improved complexity for the general case where the parameters are not required to be orthogonal.

The second problem arises from negative entries in the empirical moments tensor. LVMs for count data are constrained to have nonnegative parameters. However, the empirical moments tensor computed from the data may contain negative elements due to sampling variation and noise. Indeed, for small sample sizes or data with many small or zero counts, there will be many negative entries in the empirical moments tensor. General tensor decomposition algorithms \cite{Kuleshov15,Colombo16}, including the tensor power method (TPM) \cite{Anandkumar14}, do not guarantee the nonnegativity of model parameters. Approaches such as positive/nonnegative tensor factorization \cite{Chi12,Shashua05,Welling01} also do not address this situation as they require all the elements of the tensor to be factorized to be nonnegative. With robust tensor methods \cite{Anandkumar16,Gu14}, sparse negative entries may potentially be treated as corrupted elements; however, these methods are not applicable in this setting since there can be many negative elements.

In this paper, we introduce a novel parameter inference algorithm called partitioned tensor parallel quadratic programming (PTPQP) that is efficient in the setting where the number of variables $p$ is much larger than the number of latent components $k$. The algorithm is also robust to negative entries in the empirical moments tensor. There are two key innovations in the PTPQP algorithm. The first innovation is a partitioning technique which recovers the parameters through factorizing $O\left(p/k\right)$ much smaller sub-tensors each of size $O\left(k^3\right)$. This technique can also be combined with methods \cite{Wang15,Song16,Huang13} to obtain further improved complexities. The second innovation is a parallel quadratic programming \cite{Brand11} based algorithm to factor tensors with negative entries under the constraint that the factors are all nonnegative. To the best of our knowledge, this is the first algorithm designed to address the problem of negative entries in empirical estimator tensors. We show that the proposed factorization algorithm converges linearly with respect to each factor matrix. We also provide sufficient conditions under which the partitioned factorization scheme is consistent, the parameter estimates converge to the true parameters.

\section{Preliminaries}
\paragraph{Notations}
We use bold lowercase letters to represent vectors and bold capital letters for matrices. Tensors are denoted by calligraphic capital letters. The subscript notation $\bm{A}_j$ refers to $j$-th column of matrix $\bm{A}$. We denote the  $j$-th column
of the identity matrix as $\bm{e}_j$ and $\bm{1}$ is a vector of ones. We further write $\text{diag}\left(\bm{x}\right)$ for a diagonal matrix whose diagonal entries are $\bm{x}$, and $\text{diag}\left(\bm{A}\right)$ to mean a vector of the diagonal entries of $\bm{A}$.

Element-wise matrix operators include $\succ$ and $\succeq$, e.g., $\bm{A} \succeq 0$ means that $\bm{A}$ has nonnegative entries. $\left(\cdot\right)_+$ refers to element-wise $\max\left(\cdot,0\right)$. $*$ and $\oslash$ respectively represent element-wise multiplication and division. Moreover, $\times$ refers to the outer product and $\odot$ denotes the Khatri-Rao product. $\left\|\cdot\right\|_F$ and $\left\|\cdot\right\|_2$ represent the Frobenius norm and spectral norm, respectively. 

\paragraph{Tensor basics}
This paper uses similar tensor notations as \cite{Kolda09}. In particular, we are primarily concerned with Kruskal tensors in $\mathbb{R}^{d_1 \times d_2 \times d_3}$, which can be expressed in the form of
\begin{align}
\mathcal{T} = \sum_{j=1}^r \bm{A}_j \times \bm{B}_j \times \bm{C}_j, \label{eq:cp}
\end{align}
where $\bm{A}$, $\bm{B}$, and $\bm{C}$ are respectively $d_1$-by-$r$, $d_2$-by-$r$, and $d_3$-by-$r$ factor matrices. The rank of $\mathcal{T}$ is defined as the smallest $r$ that admits such a decomposition. The decomposition is known as the CP (CANDECOMP/PARAFAC) decomposition. The $j$-mode unfolding of $\mathcal{T}$, denoted by $\mathcal{T}_{\left(j\right)}$, for $j=1,2,3$ is a $d_j$-by-$\left(\prod_{t\neq j} d_t\right)$ matrix whose rows are serializations of the tensor fixing the index of the $j$-th dimension. The unfoldings have the following well-known compact expressions:
\begin{align}
\mathcal{T}_{\left(1\right)} = \bm{A}\left(\bm{C}\odot\bm{B}\right)^\top,
\quad
\mathcal{T}_{\left(2\right)} = \bm{B}\left(\bm{C}\odot\bm{A}\right)^\top,
\quad
\mathcal{T}_{\left(3\right)} = \bm{C}\left(\bm{B}\odot\bm{A}\right)^\top.
\label{eq:unfoldings}
\end{align}

\section{Learning through Method of Moments}
\label{sec:structure}
\subsection{Generalized Dirichlet latent variable models}
A generalized Dirichlet latent variable model (GDLM) was proposed in \cite{Zhao16} for the joint distribution of $n$ observations $\bm{y}_1,\bm{y}_2,\cdots,\bm{y}_n$. Each observation $\bm{y}_i$ consists of $p$ variables $\bm{y}_i = \left(y_{i1},y_{i2},\cdots,y_{ip}\right)^\top$. GDLM assumes a generative process involving $k$ hidden components. For each observation, sample a random Dirichlet vector $\bm{x}_i = \left(x_{i1},x_{i2},\cdots,x_{ik}\right)^\top \in \Delta^{k-1}$ with concentration parameter $\bm{\alpha} = \left(\alpha_1,\alpha_2,\cdots,\alpha_k\right)^\top$. The elements of $\bm{x}_i$ are the membership probabilities for $\bm{y}_i$ to belong to each of the $k$ components. Specifically,
\[
y_{ij} \sim \sum_{h=1}^k x_{ih} g_j \left(\theta_{jh}\right),
\]
where $g_j\left(\theta_{jh}\right)$ is the density of the $j$-th variable specific to component $h$ with parameter $\bm{\theta}_j = \left(\bm{\theta}_{j1},\bm{\theta}_{j2},\cdots, \bm{\theta}_{jk}\right)$. One advantage of GLDM is that $y_{ij}$ can take categorical values. Let $d_j$ denote the number of categories for the $j$-th variable (set $d_j = 1$ for scalar variables), $\bm{\theta}_j$ becomes a $d_j$-by-$k$ probability matrix where the $c$-th row corresponds to category $c$. We aim to accurately recover $\bm{\theta}_j$ from independent copies of $\bm{y}_i$ involving variables of mixed data types, either categorical or non-categorical.

\subsection{Moment-based estimators}
\label{sec:estimators}
The moment estimators of latent variable models typically take the form of a tensor \cite{Anandkumar14}. Consider the estimators of GDLM \cite{Zhao16} for example. Let $\bm{b}_{ij} = \bm{e}_{y_{ij}}$ if variable $j$ is categorical; $\bm{b}_{ij} = y_{ij}$ otherwise. The second- and third- order parameter estimators for variable $j$, $s$, and $t$ are written
\begin{align*}
\mathcal{M}^{js} &=
\mathbb{E}\left[\bm{b}_{ij} \times \bm{b}_{is}\right] - \frac{\alpha_0}{\alpha_0 + 1} \mathbb{E}\left[\bm{b}_{ij}\right] \mathbb{E}\left[\bm{b}_{is}\right]^\top\\
\mathcal{M}^{jst} &= \mathbb{E}\left[\bm{b}_{ij} \times \bm{b}_{is} \times \bm{b}_{it} \right]
+ \frac{2\alpha_0^2}{\left(\alpha_0 + 1\right)\left(\alpha_0 + 2\right)}  \mathbb{E}\left[\bm{b}_{ij}\right] \times \mathbb{E}\left[\bm{b}_{is}\right] \times \mathbb{E}\left[\bm{b}_{it}\right]\\
&\phantom{=} - \frac{\alpha_0}{\alpha_0 + 2} \left(
 \mathbb{E}\left[\mathbb{E}\left[\bm{b}_{ij}\right] \times \bm{b}_{is} \times \bm{b}_{it}\right]
+ \mathbb{E}\left[\bm{b}_{ij} \times \mathbb{E}\left[\bm{b}_{is}\right] \times \bm{b}_{it}\right]
+ \mathbb{E}\left[\bm{b}_{ij} \times \bm{b}_{is} \times \mathbb{E}\left[\bm{b}_{it}\right]\right]
\right).
\end{align*}
Alternatively, $\mathcal{M}^{js}$ and $\mathcal{M}^{jst}$ have the following CP decomposition into parameters $\bm{\theta}_j$:
\begin{align}
\mathcal{M}^{js} &= \sum_{h\geq 1} \frac{\alpha_h}{\alpha_0 \left(\alpha_0 + 1\right)} \bm{\theta}_{jh} \times \bm{\theta}_{sh}, \quad \bm{\theta}_{uv} \in \mathbb{R}^{d_i} \label{eq:lda-2nd}\\
\mathcal{M}^{jst} &= \sum_{h\geq 1} \frac{2\alpha_h}{\alpha_0\left(\alpha_0 + 1\right)\left(\alpha_0 + 2\right)} \bm{\theta}_{jh} \times \bm{\theta}_{sh} \times \bm{\theta}_{th}, \quad \bm{\theta}_{uv} \in \mathbb{R}^{d_i}. \label{eq:lda}
\end{align}
\cref{sec:moments} provides the derivation details of these estimators. For the special case of latent Dirichlet allocation, $\mathcal{M}^{js}$ and $\mathcal{M}^{jst}$ are scalar joint probabilities.

The parameters $\bm{\theta}_j$ are typically obtained by factorizing the block tensor $\mathcal{M}_2$ whose $\left(j,s\right)$-th element is the empirical $\widehat{\mathcal{M}}^{js}$ and/or $\mathcal{M}_3$ whose $\left(j,s,t\right)$-th element is the empirical $\widehat{\mathcal{M}}^{jst}$ \cite{Anandkumar16,Anandkumar14,Zhao16}. Note that $\bm{\theta}_j$ are generally non-orthogonal, and thus preprocessing steps (see \cref{sec:tpm_preproc}) are needed for orthogonal decomposition methods \cite{Wang15,Song16,Anandkumar14}. The preprocessing can be expensive and often leads to suboptimal performance \cite{Souloumiac09,Colombo16}. Here, we highlight a few relevant observations:
\begin{itemize}
\item $\mathcal{M}^{js}$ alone does not yield unique parameters $\bm{\theta}_j$ due to the well-known rotation problem. Suppose that $\bm{\theta}_j^*$ and $\bm{\theta}_s^*$ are the ground-truth parameters satisfying \eqref{eq:lda-2nd} and any invertible $\bm{R}$, there exists decomposition $\bm{\theta}_j^\prime = \bm{\theta}_j^* \bm{R}$ and $\bm{\theta}_s^\prime = \bm{R}^{-1} \bm{\theta}_s^*$ that also satisfy \eqref{eq:lda-2nd} but are not ground-truth parameters.
 The ground-truth parameters are not uniquely identifiable through $\mathcal{M}^{js}$, this is true even when enforcing nonnegativity constraints on parameters \cite{Donoho04}.
\item $\mathcal{M}^{jst}$ is sufficient to uniquely recover the parameters under certain mild conditions \cite{Kruskal77}; for example, when any two of $\bm{\theta}_j$, $\bm{\theta}_s$, and $\bm{\theta}_t$ have linearly independent columns and the columns of the third are pair-wise linearly independent \cite{Leurgans93}.
\item The empirical estimator $\mathcal{\widehat{M}}^{jst}$ generally contains negative entries due to variance and noise. The fraction of negative entries can approach 50\%, as we shall see in experiments. We address this issue in \cref{sec:neg}.
\item While the decomposition \eqref{eq:lda} can be unique up to permutation and rescaling, the correspondence between each column of the factor matrix and each hidden component may not be consistent across multiple decompositions. Techniques for achieving consistency are developed in \cref{sec:matching}.
\end{itemize}

\subsection{Computational complexity}
\label{sec:complexity}
Tensor methods such as TPM typically decompose the $O\left(p^3 d_{max}^3\right)$ full estimator tensor that includes all variables. More efficient algorithms have been developed for the case that parameters are orthogonal \cite{Wang15,Song16}, and when the sample size is small \cite{Huang13}. However, these methods do not apply in the general case where the parameters are non-orthogonal and the sample size can be potentially large. A key insight underlying our approach is that it is sufficient to recover the parameters by factorizing only $O\left(p/k\right)$ much smaller sub-tensors each of size $O\left(k^3\right)$. This technique can also be combined with the aforementioned methods to further improve the complexity in certain cases.

\section{An efficient algorithm}
In this section, we develop partitioned tensor parallel quadratic programming (PTPQP) an efficient approximate algorithm for learning mixed membership models. We first introduce a novel partitioning-and-matching scheme that reduces parameter estimation to factorizing a sequence of sub-tensors. Then, we develop a nonnegative factorization algorithm that can handle negative entries in the sub-tensors.

\subsection{Partitioned factorization}
\label{sec:partition}
Factorizing the full tensor formed by all $\mathcal{M}^{jst}$ is expensive while a three-variable tensor $\mathcal{M}^{jst}$ in \eqref{eq:lda} alone may not be sufficient to determine $\bm{\theta}_j$ when $k$ is large. In this section, we consider factorizing the sub-tensors corresponding to a cover of the set of variables $\left[p\right]$ such that each sub-tensor admits an identifiable CP decomposition \eqref{eq:cp}, i.e. unique up to permutation and rescaling of columns. This gives the parameters for all variables. Suppose that $p > k$ and the maximum number of categories $d_{\text{max}}$ is a constant, the aggregated size of the sub-tensors can be much smaller, i.e., $O\left(p k^2\right)$, than the size $O\left(p^3\right)$ of the full estimator.

Let $\pi^j$, $\pi^s$, and $\pi^t$ denote ordered subsets $\subseteq \left[p\right]$, with cardinality $\left|\pi_j\right| = p_j$, $\left|\pi_s\right| = p_s$, and $\left|\pi_t\right| = p_t$, respectively. Consider the $p_j$-by-$p_s$-by-$p_t$ block tensor \footnote{For block tensor operations, see e.g., \cite{Ragnarsson12}.} $\mathcal{M}^{\pi^j\pi^s\pi^t}$ whose $\left(u,v,w\right)$-th element is the tensor $\mathcal{M}_{uvw}^{\pi^j\pi^s\pi^t} = \mathcal{M}^{\pi_u^j\pi_v^s\pi_w^t}$. From \eqref{eq:lda}, we have that
\begin{align}
\mathcal{\bm{M}}^{\pi^j\pi^s\pi^t} =
\sum_{h=1}^k
\frac{2\alpha_h}{\alpha_0\left(\alpha_0 + 1\right)\left(\alpha_0 + 2\right)} 
\begin{bmatrix}
\bm{\theta}_{{\pi_1^j}h}\\
\bm{\theta}_{{\pi_2^j}h}\\
\vdots\\
\bm{\theta}_{{\pi_{p_j}^j}h}
\end{bmatrix}
\times
\begin{bmatrix}
\bm{\theta}_{{\pi_1^s}h}\\
\bm{\theta}_{{\pi_2^s}h}\\
\vdots\\
\bm{\theta}_{{\pi_{p_s}^s}h}
\end{bmatrix}
\times
\begin{bmatrix}
\bm{\theta}_{{\pi_1^t}h}\\
\bm{\theta}_{{\pi_2^t}h}\\
\vdots\\
\bm{\theta}_{{\pi_{p_t}^t}h}
\end{bmatrix}. \label{eq:part-decomp}
\end{align}
Clearly, the block tensor is identifiable if it has an identifiable sub-tensor. Suppose that a sub-tensor $\mathcal{\bm{M}}^{\pi^u\pi^v\pi^w}$ is identifiable, then one can construct an identifiable tensor $\mathcal{\bm{M}}^{\pi^{j\prime}\pi^{s\prime}\pi^{t\prime}}$ from  $\mathcal{\bm{M}}^{\pi^j\pi^s\pi^t}$ by setting 
\begin{align}
\pi^{j\prime} = \pi^j \cup \pi^u, \quad \pi^{s\prime} = \pi^s \cup \pi^v, \quad \pi^{t\prime} = \pi^t \cup \pi^w. \label{eq:aug}
\end{align}
We further remark that a sub-tensor can be identifiable under mild conditions, for example, if the sum of the Kruskal rank of the three factor matrices is at least than $2 k + 2$ \cite{Kruskal77}. 

Given an identifiable sub-tensor $\mathcal{M}^{\pi^u\pi^v\pi^w}$ of {\em anchor variables} indexed by $\pi^u$, $\pi^v$, and $\pi^w$, the partitioning produces a set of sub-tensors (partitions) constructed through \eqref{eq:aug}, that includes all variables. Thus, $\mathcal{M}^{\pi^u\pi^v\pi^w}$ is a common sub-tensor shared across all partitions. We choose anchor variables whose parameter matrices are of full column rank to obtain an identifiable $\mathcal{M}^{\pi^u\pi^v\pi^w}$. Finally, one can divide the rest of variables evenly and randomly into the partitions.

\subsection{Matching parameters with hidden components}
\label{sec:matching}
Since the factorization of a partition \eqref{eq:part-decomp} can only be identifiable up to permutation and rescaling of the columns of constituent $\bm{\theta}_j$, the correspondence between the columns of $\bm{\theta}_j$ and hidden components can differ across partitions. To enforce consistency, we associate a permutation operator $\psi^j$ for each variable $j$ such that $\left(\psi^j \bm{\theta}_j\right)_h$ are the parameters specific to hidden component $h$ across all variables $j$. Consider the following vector representation of $\psi$:
\begin{align*}
\psi & = \left(\psi_1, \psi_2, \cdots, \psi_k\right), \quad \psi_i \in \left[k\right]\\
\psi \bm{A} &= \left[\bm{A}_{\psi_1}, \bm{A}_{\psi_1}, \cdots, \bm{A}_{\psi_k}\right].
\end{align*}
Observe that $\psi^j=\psi^s=\psi^t$ within a factorization of $\mathcal{M}^{jst}$, and this also holds for the partitioned factorization \eqref{eq:part-decomp} of $\mathcal{\mathcal{M}}^{\pi^j\pi^s\pi^t}$ as well, i.e., $\psi^x = \psi^y, \quad \forall x, y \in \pi^j\cup\pi^s\cup\pi^t.$

Consider the factorizations of $\mathcal{\mathcal{M}}^{\pi^j\pi^s\pi^t}$ and $\mathcal{\mathcal{M}}^{\pi^u\pi^v\pi^w}$ and suppose that $\exists x\in \left(\pi^j\cup\pi^s\cup\pi^t\right) \cap \left(\pi^u\cup\pi^v\cup\pi^w\right)$. The permutation
operator for one factorization is determined given the other by column matching the parameters of variable $x$ in both factorizations. Thus, an inductive way to achieve a consistent factorization is to start with one factorization, and let its permutation be the identity $\left(1,2,\cdots,k\right)$, then perform the factorization over new sets of variables with at least one variable in common with the initial factorization. Permutations for the sequential factorizations are determined via column matching parameter matrices of the common variables.

Given two factorized parameter matrices $\bm{\theta}_j$ and $\bm{\theta}_j^\prime$ of variable $j$, our goal is to find a {\em consistent permutation} $\psi$ (of $\bm{\theta}_j$ with respect to $\bm{\theta}_j^\prime$) such that $\left(\psi \bm{\theta}_j\right)_h$ and $\bm{\theta}_{jh}^\prime$ correspond to the same hidden component for all $h \in \left[k\right]$. We now present an algorithm with provable guarantees to compute a consistent permutation.

\paragraph{Smallest angle matching}
\label{sec:angle}
A simple matching algorithm is to match the two columns of the two parameter matrices that have the smallest angle between them. Consider the factorizations of $\mathcal{M}^{jst}$ and $\mathcal{M}^{juv}$ which yield respectively parameters $\bm{\theta}_j$ and $\bm{\theta}_j^\prime$ for the common variable $j$. Given the permutation $\psi^j$ for $\mathcal{M}^{jst}$, the permutation $\psi^u$ for $\mathcal{M}^{juv}$ is computed by:
\begin{align}
\psi_s^u = \arg \max_t \left( \bm{\bar{\theta}}_j^{\prime\top} \psi^j \bm{\bar{\theta}}_j \right)_{ts} \label{eq:sam}.
\end{align}
Here, $\bm{\bar{\theta}}_j$ and $\bm{\bar{\theta}}_j^\prime$ represent respectively the normalized $\bm{\theta}_j$ and $\bm{\theta}_j^\prime$ with each column having unit Euclidean norm. 

There are cases that $\psi^u$ computed via \eqref{eq:sam} is not consistent: 1) $\psi^u$ contains duplicate entries and hence is ineligible; and 2) since $\bm{\theta}_j$ and $\bm{\theta}_j^\prime$ are the factorized parameter matrices which are generally perturbed from the ground-truth, the resulting $\psi^u$ may differ from the consistent permutation. To cope with these cases, we establish in \cref{sec:guarantees} the sufficient conditions for $\psi^u$ to be consistent.

\paragraph{Orthogonal Procrustes matching}
\label{sec:procrustes}
One issue with the smallest angle matching is that each column is paired independently. It is easy for multiple columns to be paired with a common nearest neighbor. We describe a more robust algorithm based on the orthogonal Procrustes problem, and show improved guarantees. Since a consistent permutation is orthogonal, a natural relaxation is to only require the operator to be orthogonal. This is an orthogonal Procrustes problem, formulated in the same settings as \cref{sec:angle}
\begin{align}
\min_{\bm{\Psi}} \left\|\bm{\bar{\theta}}_j^\prime \bm{\Psi} - \psi^j \bm{\bar{\theta}}_j\right\|_F^2,
\quad \text{s.t.}\;
\bm{\Psi}^\top \bm{\Psi} = \bm{I}. \label{eq:procrustes}
\end{align}
Let $\bm{\bar{\theta}}_j^{\prime\top} \psi^j \bm{\bar{\theta}}_j = \bm{U} \bm{\Sigma} \bm{V}^\top$ be the singular value decomposition (SVD), the solution $\bm{\Psi}^*$ is given by the polar factor \cite{Procrustes66}
\begin{align}
\bm{\Psi}^* &= \bm{U} \bm{V}^\top. \label{eq:procrustes-sol}
\end{align}
Here, $\bm{\Psi}^*$ is orthogonal and does not immediately imply the desired permutation $\psi^u$. To compute $\psi^u$, one can additionally restrict $\bm{\Psi}$ to be a permutation matrix, and solve for $\psi^u$ using linear programming \cite{Gower04}. Aside from efficiency, one fundamental question is that under what assumptions the objective \eqref{eq:procrustes} yields the consistent permutation.

Given the solution $\bm{\Psi}^*$ to the Procrustes problem, we propose the following simple algorithm for computing $\psi^u$:
\begin{align}
\psi_s^u = \arg \max_t \bm{\Psi}_{ts}^*
\label{eq:opm}.
\end{align}
We first establish through \cref{thm:max-opt} that if $\psi^u$ obtained using \eqref{eq:opm} is a valid permutation, i.e., no duplicate entries, then it is optimal in terms of the objective \eqref{eq:procrustes}.
\begin{theorem}
\label{thm:max-opt}
The $\psi^u$ obtained using $\eqref{eq:opm}$ satisfies
\begin{align*}
\left\|\psi^u\bm{\bar{\theta}}_j^\prime - \psi^j \bm{\bar{\theta}}_j\right\|_F^2 \leq \left\|\psi\bm{\bar{\theta}}_j^\prime - \psi^j \bm{\bar{\theta}}_j\right\|_F^2
\end{align*}
for all permutations $\psi$.
\end{theorem}
\begin{proof}
First, rewrite the objective \eqref{eq:procrustes} as follows
\begin{align}
\left\|\bm{\bar{\theta}}_j^\prime \bm{\Psi} - \psi^j \bm{\bar{\theta}}_j\right\|_F^2
&= \tr\bm{\Psi}^\top \bm{\bar{\theta}}_j^{\prime\top} \bm{\bar{\theta}}_j^\prime \bm{\Psi}
+ \tr\left(\psi^j\bm{\bar{\theta}}_j\right)^\top \psi^j\bm{\bar{\theta}}_j
- 2\tr \bm{\Psi}^\top \bm{\bar{\theta}}_j^{\prime\top}  \psi^j\bm{\bar{\theta}}_j \notag\\
&= \left\|\bm{\bar{\theta}}_j^\prime\right\|_F^2 + \left\|\bm{\bar{\theta}}_j\right\|_F^2
- 2\tr \bm{\Psi}^\top \bm{\bar{\theta}}_j^{\prime\top}  \psi^j\bm{\bar{\theta}}_j. \label{eq:procrustes-trace}
\end{align}
Recall the SVD $\bm{\bar{\theta}}_j^{\prime\top} \psi^j\bm{\bar{\theta}}_j = \bm{U}\bm{\Sigma}\bm{V}^\top$, and write $\bm{\Psi} = \bm{U}\bm{V}^\top + \bm{E}$. Keeping only terms that depend on $\bm{E}$ in \eqref{eq:procrustes-trace} to obtain $-2\tr \bm{E}^\top \bm{U} \bm{\Sigma} \bm{V}^\top$. Thus, the optimization \eqref{eq:procrustes} is equivalent to
\begin{align*}
\max_{\bm{E}} \tr \bm{V}^\top \bm{E}^\top \bm{U} \bm{\Sigma},
\quad \text{s.t.} \quad
\left(\bm{U}\bm{V}^\top + \bm{E}\right)^\top \left(\bm{U}\bm{V}^\top + \bm{E}\right) = \bm{I}.
\end{align*}
From the constraint, we obtain $\tr \bm{E}^\top\bm{E} = -2\tr \bm{V}^\top \bm{E}^\top \bm{U}$. The optimization now becomes
\begin{align}
\min_{\bm{E}} \tr \bm{E}^\top\bm{E}\bm{\Sigma} = \min_{\bm{E}} \sum_j \left(\bm{E}^\top \bm{E}\right)_{jj} \bm{\Sigma}_{jj}. \label{eq:procrustes-err}
\end{align}
Let us now restrict each column of $\bm{\Psi}$ to be in $\left\{ \bm{e}_j \;|\; j=1,2,\cdots,k \right\}$, but not necessarily distinct. Suppose that $\bm{\Psi}_j = \bm{e}_y$. We have that $\bm{E}_j = \bm{e}_y - \left(\bm{U}\bm{V}^\top\right)_j$. Clearly, \eqref{eq:procrustes-err} and hence \eqref{eq:procrustes} are minimized with $y = \arg\max_t \left(\bm{U}\bm{V}^\top\right)_{tj}$. 
\end{proof}

In section \cref{sec:guarantees} we state sufficient conditions under which the objective \eqref{eq:procrustes} yields a consistent permutation.

\subsection{Approximate nonnegative factorization}
\label{sec:objs}
In previous sections, we reduced the inference problem to factorizing partitioned sub-tensors. We now present a factorization algorithm for the sub-tensors that contain negative entries. Our goal is to approximate a sub-tensor $\mathcal{M}$ by a sub-tensor $\widetilde{\mathcal{M}} = \sum_j \bm{A}_j \times \bm{B}_j \times \bm{C}_j$ where the factors $\bm{A}$, $\bm{B}$, and $\bm{C}$ are nonnegative. The Frobenius norm is used to quantify the approximation
\begin{align}
 \min_{\bm{A}, \bm{B}, \bm{C} \succeq 0}
 \left\|\mathcal{M} - \widetilde{\mathcal{M}}\right\|_F.
 \label{eq:ls-obj}
\end{align}
Note that we do not assume that $\mathcal{M} \succeq 0$ in \eqref{eq:ls-obj} which distinguishes our optimization problem 
from other approximate factorization algorithms \cite{Welling01,Chi12,Kim14,Shashua05}. In \cref{sec:issue-neg}, we provide some details as to why negative entries are problematic for standard approximate factorization algorithms.
We can rewrite \eqref{eq:ls-obj} using the 1-mode unfolding as
\begin{align}
\min_{\bm{A}, \bm{B}, \bm{C} \succeq 0} & \left\|\mathcal{M}_{\left(1\right)} - \bm{A}\left(\bm{C}\odot \bm{B}\right)^\top\right\|_F \label{eq:ls-obj-mode}.
\end{align}
Equivalent formulations with respect to the 2-mode and 3-mode unfoldings can be readily obtained from \eqref{eq:unfoldings}. 

We point out that another widely-used error measure --- the I-divergence \cite{Finesso06,Chi12} --- may not be suitable for our learning problem. The optimization using I-divergence is given by
\begin{align*}
\min_{\bm{A}, \bm{B}, \bm{C} \succeq 0}\sum_{u,v,w} & \left[\mathcal{M}_{uvw} \log \frac{\mathcal{M}_{uvw}}{\widetilde{\mathcal{M}}_{uvw}} - \mathcal{M}_{uvw} + \widetilde{\mathcal{M}}_{uvw}\right].
\end{align*}
This optimization is useful for nonnegative $\mathcal{M}$ when each entry follows a Poisson distribution. In this case, the objective is equivalent to the sum of Kullback-Leibler divergence across all entries of $\mathcal{M}$:
\begin{align*}
\sum_{u,v,w} D_{\text{KL}}\left(\text{Pois}\left(x; \mathcal{M}_{uvw}\right) \big\|\; \text{Pois}\left(x;\widetilde{\mathcal{M}}_{uvw}\right)\right).
\end{align*}
However, the Poisson assumption does not generally hold for the estimator tensor \eqref{eq:lda}.

\subsection{Handling negative entries in empirical estimators}
\label{sec:neg}
We first illustrate that factorizing a tensor with negative entries using either positive tensor factorization \cite{Welling01} or nonnegative tensor factorization \cite{Chi12,Shashua05} will either result in factors that violate 
the the nonnegativity constraint or the result of the algorithm diverges. In addition, we show that general tensor decompositions cannot enforce the factor nonnegativity even after rounding the negative entries to zero.

We then present a simple method based on weighted nonnegative matrix factorization (WNMF) \cite{Zhang96} that enforce the factor nonnegativity constraint. We further generalize this method using parallel quadratic programming (PQP) \cite{Brand11} to obtain a method with a provable convergence rate.

\paragraph{Issue of negative entries}
\label{sec:issue-neg}
If the tensor is strictly nonnegative, the optimization specified in \eqref{eq:ls-obj} can be reduced to nonnegative matrix factorization (NMF). Solvers abound for NMF including the celebrated Lee-Seung's multiplicative updates \cite{Lee01}. The reduction is done by viewing \eqref{eq:ls-obj-mode} as $\left\|\bm{Y} - \bm{W} \bm{H} \right\|_F^2$ with $\bm{Y} = \mathcal{M}_{\left(1\right)}^{jst}$, $\bm{W} = \bm{A}$, and $\bm{H} = \left(\bm{C}\odot \bm{B}\right)^\top$, and alternating
\begin{align}
W_{st} \leftarrow W_{st} \frac{\left(\bm{Y} \bm{H}^\top\right)_{st}}{\left(\bm{W}\bm{H}\bm{H}^\top\right)_{st}},
\label{eq:leeseung}
\end{align}
over each unfolding and factor matrix $\bm{W}$. Obviously, the updates may yield negative entries in $\bm{W}$ when the unfolding contains negative entries. In addition, convergence relies on the nonnegativity of the unfolding \cite{Lee01}. This issue extends to their tensor factorization variants \cite{Welling01,Chi12,Kim14} known as the positive tensor factorization and nonnegative tensor factorization. For these approaches, a naive resolution is to round negative entries of $\mathcal{\widehat{M}}^{jst}$ to $0$, this however lacks theoretical guarantees.

It is important to note that the rounding does not help general tensor decompositions like TPM. The following example illustrates that the unique decomposition (up to permutation and rescaling) of a positive tensor can contain negative entries. Consider a $2$-by-$2$-by-$2$ positive tensor, whose $1$-mode unfolding is given by
\begin{align*}
\begin{bmatrix}[c c | c c]
1 & 3 & 2 & 2\\
2 & 2 & 2 & 2
\end{bmatrix},
\end{align*}
where the vertical bar separates two frontal slices. It has the following decomposition, written in the form of \eqref{eq:cp}:
\begin{align*}
\bm{A} = \bm{C} = \begin{bmatrix}
1 & 1\\
1 & 0
\end{bmatrix}, \quad
\bm{B} = \begin{bmatrix}
2 & -1\\
2 & 1
\end{bmatrix}.
\end{align*}
Since all factors are of full-rank, the decomposition is unique up to permutation and rescaling of columns \cite{Kruskal77}. Thus, a general tensor decomposition yields a $\bm{B}$ with negative entries regardless of rescaling.

\subsection{Factorization via WNMF}
Since the ground-truth $\mathcal{M}^{jst}$ are nonnegative, we may ``ignore" the negative entries of $\widehat{\mathcal{M}}^{jst}$ by treating them as missing values. This idea leads to the following modified objective:
\begin{align}
\min_{\bm{W},\bm{H} \succeq 0} \left\|\bm{\Omega} * \left(\bm{Y} - \bm{W} \bm{H}\right)\right\|_F^2 \label{eq:wnmf}
\end{align}
where $\bm{Y}$, $\bm{W}$, $\bm{H}$ are chosen identically as \eqref{eq:leeseung}, and we define
\begin{align*}
\Omega_{uv} = \begin{cases}
1, \quad Y_{uv} \geq 0\\
0, \quad Y_{uv} < 0
\end{cases}.
\end{align*}
The optimization can be carried out using WNMF. Here, we modify the original updates by introducing a positive constant $\epsilon$ to ensure that the updates are well-defined:
\begin{align}
W_{uv} \leftarrow W_{uv} \frac{\left[\left(\bm{\Omega}*\bm{Y}\right) \bm{H}^\top\right]_{uv} + \epsilon}{\left[\left(\left(\bm{W}\bm{H}\right)*\bm{\Omega}\right) \bm{H}^\top\right]_{uv} + \epsilon},
\quad
H_{uv} \leftarrow H_{uv} \frac{\left[\bm{W}^\top \left(\bm{\Omega}*\bm{Y}\right) \right]_{uv} + \epsilon}{\left[\bm{W}^\top \left(\bm{\Omega}*\left(\bm{W}\bm{H}\right)\right) \right]_{uv} + \epsilon}.
\label{eq:wnmf-updates}
\end{align}
\cref{thm:wnmf-updates} states the correctness of the modified updates \eqref{eq:wnmf-updates}.

\begin{theorem}
\label{thm:wnmf-updates}
The objective \eqref{eq:wnmf} is non-increasing under the multiplicative updates \eqref{eq:wnmf-updates}.
\end{theorem}
\begin{proof}
We prove the update for $\bm{H}$, and the update for $\bm{W}$ follows by applying the update to $\left\|\bm{\Omega}^\top * \left(\bm{v}^\top - \bm{H}^\top \bm{W}^\top\right)\right\|_F$. First, consider the error Frobenius norm for a column $\bm{h}$ of $\bm{H}$, and the corresponding columns $\bm{\omega}$ of $\bm{\Omega}$ and $\bm{v}$ of $\bm{V}$,
\begin{align*}
F\left(\bm{h}\right) = \left\|\bm{\omega} * \left(\bm{v} - \bm{W} \bm{h}\right)\right\|_F^2.
\end{align*}
The following $G\left(\cdot,\cdot\right)$ is an auxiliary function of $F\left(\cdot\right)$:
\begin{align}
G\left(\bm{h},\bm{h}^t\right) = F\left(\bm{h}\right) + \left(\bm{h} - \bm{h}^t\right)^\top \nabla F\left(\bm{h}^t\right) + \frac{1}{2} \left(\bm{h} - \bm{h}^t\right)^\top \bm{K} \left(\bm{h} - \bm{h}^\top \right), \label{eq:aux}
\end{align}
where we define
\begin{align*}
\bm{K} = \text{diag}\left(\left[
\bm{W}^\top \text{diag}\left(\bm{\omega}\right) \bm{W} \bm{h}^t + \epsilon \bm{1}
\right] \oslash \bm{h}^t
\right).
\end{align*}
Clearly, $G\left(\bm{h},\bm{h}\right) = F\left(\bm{h}\right)$, and one can show that $G\left(\bm{h},\bm{h}^t\right) \geq F\left(\bm{h}\right)$ by rewriting
\begin{align}
F\left(\bm{h}\right) = F\left(\bm{h}\right) + \left(\bm{h} - \bm{h}^t\right)^\top \nabla F\left(\bm{h}^t\right) + \frac{1}{2} \left(\bm{h} - \bm{h}^t\right)^\top \bm{W}^\top \text{diag}\left(\bm{\omega}\right) \bm{W} \left(\bm{h} - \bm{h}^\top \right), \label{eq:taylor}
\end{align}
where we note that $\bm{\omega}*\bm{\omega} = \bm{\omega}$ from the Boolean definition of $\bm{\omega}$. Comparing \eqref{eq:aux} with \eqref{eq:taylor}, it is sufficient to show that $\bm{K} - \bm{W}^\top \text{diag}\left(\bm{\omega}\right) \bm{W}$ is positive semi-definite. Now consider the scaled matrix
\begin{align*}
\bm{U} &= \text{diag}\left(\bm{h}^t\right) \bm{K} \text{diag}\left(\bm{h}^t\right) - 
\text{diag}\left(\bm{h}^t\right) \bm{W}^\top \text{diag}\left(\bm{\omega}\right) \bm{W} \text{diag}\left(\bm{h}^t\right)\\
&= \text{diag}\left(\bm{W}^\top \text{diag}\left(\bm{\omega}\right) \bm{W} \bm{h}^t + \epsilon \bm{1}\right) \text{diag}\left(\bm{h}^t\right) - 
\text{diag}\left(\bm{h}^t\right) \bm{W}^\top \text{diag}\left(\bm{\omega}\right) \bm{W} \text{diag}\left(\bm{h}^t\right).
\end{align*}
Observe that $\bm{U}$ is strictly diagonally dominant as $\bm{U}\bm{1} \succ 0$ and the off-diagonal entries are negative. Also note that all diagonal entries of $\bm{U}$ are positive, it follows that $\bm{U}$ is positive semi-definite. We thereby conclude that $\bm{K} - \bm{W}^\top \text{diag}\left(\bm{\omega}\right) \bm{W}$ is positive semi-definite.

Let $\bm{h}^{t+1} = \arg \min_{\bm{h}} G\left(\bm{h},\bm{h}^t\right)$, we have that $F\left(\bm{h}^t\right) = G\left(\bm{h}^t,\bm{h}^t\right) \geq G\left(\bm{h}^{t+1},\bm{h}^t\right) \geq F\left(\bm{h}^{t+1}\right)$. The minimizer $\bm{h}^{t+1}$ is obtained by setting $\nabla_{\bm{h}^{t+1}} G\left(\bm{h}^{t+1},\bm{h}^t\right) = 0$, which yields
\begin{align*}
- \nabla F\left(\bm{h}^t\right) &= \bm{K} \left(\bm{h}^{t+1} - \bm{h}^t\right)\\
\bm{W}^\top \left[\bm{\omega}*\left(\bm{v} - \bm{W} \bm{h}^t\right)\right] 
&= \bm{K} \bm{h}^{t+1} - \bm{W}^\top \text{diag}\left(\bm{\omega}\right) \bm{W} \bm{h}^t - \epsilon \bm{1}\\
\bm{h}^{t+1} &= \bm{h}^t * \left[\bm{W}^\top \left(\bm{\omega}*\bm{v}\right) + \epsilon\bm{1} \right] \oslash
\left[\bm{W}^\top \left(\bm{\omega} * \left(\bm{W} \bm{h}^t\right)\right) + \epsilon \bm{1}\right].
\end{align*}
The particular choice of $\bm{\Omega}$ guarantees that $\bm{h}^{t+1}$ is always positive.
\end{proof}

\subsection{Parallel quadratic programming}
We now generalize the WNMF approach using parallel quadratic programming to obtain a convergence rate. Let $\mathbb{S}_{++}$ denote the set of symmetric positive definite matrices, we consider the following optimization problem
\begin{align}
\min_{\bm{x}} \frac{1}{2} \bm{x}^\top \bm{Q} \bm{x} + \bm{z}^\top \bm{x} \quad \text{s.t.} \quad \bm{x} \geq 0, \; \bm{Q} \in \mathbb{S}_{++}, \label{eq:nqp}
\end{align}
which can be solved by iterating multiplicative updates \cite{Brand11,Sha03}. We use the parallel quadratic programming (PQP) algorithm \cite{Brand11,Brand11a} to solve \eqref{eq:nqp}, partly because it has a provable linear convergence rate. The PQP multiplicative update for \eqref{eq:nqp} takes the following simple form:
\begin{align}
\bm{x} \leftarrow \bm{x} * \left(\bm{Q}^- \bm{x} + \bm{z}^-\right) \oslash \left(\bm{Q}^+ \bm{x} + \bm{z}^+\right), \label{eq:pqp}
\end{align}
with
\begin{align*}
\bm{Q}^+ &= \left(\bm{Q}\right)_+ + \text{diag}\left(\bm{\gamma}\right), \quad &\bm{Q}^- &= \left(-\bm{Q}\right)_+ + \text{diag}\left(\bm{\gamma}\right)\\
\bm{z}^+ &= \left(\bm{z}\right)_+ + \bm{\phi}, \quad &\bm{z}^- &= \left(-\bm{z}\right)_+ + \bm{\phi}.
\end{align*}
Here $\bm{\gamma}$ and $\bm{\phi}$ are arguments to PQP, we will discuss these arguments in section \cref{sec:convergence}. The update maintains nonnegativity since all items are nonnegative. We make the following observation.

\begin{theorem}
\label{thm:leeseung-pqp}
The multiplicative updates for Lee-Seung and WNMF are special cases of PQP.
\end{theorem}
\begin{proof}
Since the WNMF \eqref{eq:wnmf} generalizes Lee-Seung, which is the case that $\bm{\Omega}$ has all ones, we need only to prove for WNMF. Let $\bm{\Lambda} = \bm{\Omega}*\bm{\Omega}$ and $\bm{\gamma} = 0$, some matrix algebra reveals the following PQP updates
\begin{align}
\begin{aligned}
W_{uv} &\leftarrow W_{uv} \frac{\left[\left(\bm{\left(\Lambda}*\bm{Y}\right) \bm{H}^\top\right)_+\right]_{uv} + \Phi_{uv}}
{\left[\left(\left(\bm{W}\bm{H}\right)*\bm{\Lambda}\right) \bm{H}^\top + \left(\left(-\bm{\Lambda}*\bm{Y}\right) \bm{H}^\top\right)_+ \right]_{uv} + \Phi_{uv}}\\
H_{uv} &\leftarrow H_{uv} \frac{\left[\left(\bm{W}^\top \left(\bm{\Lambda}*\bm{Y}\right)\right)_+\right]_{uv} + \Phi_{uv}^\prime}
{\left[\bm{W}^\top \left(\bm{\Lambda}*\left(\bm{W}\bm{H}\right) \right) + \left(-\bm{W}^\top \left(\bm{\Lambda}*\bm{Y}\right) \right)_+ \right]_{uv} + \Phi_{uv}^\prime}. \label{eq:pqp-wls}
\end{aligned}
\end{align}
Comparing \eqref{eq:pqp-wls} to \eqref{eq:wnmf-updates}, they are equivalent if $\Phi_{uv} = \Phi_{uv}^\prime = \epsilon$.
\end{proof}

We can now solve the approximate nonnegative factorization problem stated in \eqref{eq:ls-obj} using \eqref{eq:pqp}. \cref{thm:updates} states the multiplicative updates. A more detailed discussion of $\bm{\Phi}$ is included in \cref{sec:convergence}. We present pseudo-code in 
\cref{alg:tpqp}.

\begin{theorem}
\label{thm:updates}
For optimization \eqref{eq:ls-obj}, the following update converges linearly to a local optimum
\begin{align}
\bm{A} \leftarrow \bm{A} * \left[\left(-\bm{Z}\right)_+ + \bm{\Phi}\right] \oslash \left[\bm{A}\bm{Q} + \left(\bm{Z}\right)_+ + \bm{\Phi}\right] \label{eq:tpqp}
\end{align}
with
\begin{align*}
\bm{Q} = \left(\bm{C}^\top \bm{C}\right)*\left(\bm{B}^\top\bm{B}\right), 
\qquad
\bm{Z} = -\mathcal{M}_{\left(1\right)}\left(\bm{C}\odot \bm{B}\right)\\
\bm{\Phi} \succ
\frac{1}{2} \left(
\sqrt{
\frac{\text{diag}\left(\bm{Z}\bm{Q}^{-1}\bm{Z}^\top\right)}
{\lambda_{\text{min}}\left(\bm{Q}\right)}}
\text{diag}\left(\bm{Q}\right)^\top
- \left|\bm{Z}\right|
\right)_+,
\end{align*}
where $\lambda_{\text{min}}\left(\cdot\right)$ is the smallest eigenvalue. Similar updates for $\bm{B}$ and $\bm{C}$ are obtained using \eqref{eq:unfoldings}.
\end{theorem}
\begin{proof}
We apply PQP updates \eqref{eq:pqp} to each row of $\bm{A}$. Let $\bm{v}_{j:}$ and $\bm{A}_{j:}$ be the $j$-th row of $\mathcal{M}_{\left(1\right)}$ and $\bm{A}$, respectively. Fixing the current factor estimates $\bm{B}$ and $\bm{C}$, the optimization with respect to $\bm{A}_{j:}$ follows from \eqref{eq:ls-obj-mode}:
\begin{align*}
\arg\min_{\bm{A}_{j:} \succeq 0} \left\|\bm{v}_{j:}^\top - \left(\bm{C}\odot \bm{B}\right)\bm{A}_{j:}^\top\right\|_F 
= \arg\min_{\bm{A}_{j:} \succeq 0} 
\frac{1}{2} \bm{A}_{j:} \left(\bm{C}\odot \bm{B}\right)^\top \left(\bm{C}\odot \bm{B}\right) \bm{A}_{j:}^\top
- \bm{v}_{j:} \left(\bm{C}\odot \bm{B}\right) \bm{A}_{j:}^\top.
\end{align*}
Now the updates \eqref{eq:pqp} can be applied immediately, where we set $\bm{\gamma} = 0$ and $\bm{\Phi}$ according to \cref{thm:convergence} in \cref{sec:guarantees}. Using the identity $\left(\bm{C}\odot \bm{B}\right)^\top \left(\bm{C}\odot \bm{B}\right) =  \left(\bm{C}^\top \bm{C}\right)*\left(\bm{B}^\top\bm{B}\right)$ and performing the updates simultaneous for all rows of $\bm{A}$ give \eqref{eq:tpqp}.

\end{proof}

\begin{algorithm}[tb]
\caption{ Factorize $\left(\mathcal{M}, k, \bm{d}\right)$ }
\label{alg:tpqp}
\begin{algorithmic}
\STATE {$\mathcal{M} \leftarrow \mathcal{M} / \max_{uvw} \left|\mathcal{M}_{uvw}\right|, \quad \epsilon \leftarrow 10^{-10}$}
\STATE {{\em \% Initialize with random nonnegative matrices:}}
\STATE {$\bm{A} \leftarrow \text{rand}\left(d_j,k\right)$, $\bm{B} \leftarrow \text{rand}\left(d_s,k\right)$, $\bm{C} \leftarrow \text{rand}\left(d_t,k\right)$}
\STATE {{\em \% Create a set of alternating variable tuples:}}
\STATE{$F \leftarrow \left\{\left[\bm{A},
\left(\bm{C}^\top \bm{C}\right)*\left(\bm{B}^\top\bm{B}\right),
-\mathcal{M}_{\left(1\right)} \left(\bm{C}\odot \bm{B}\right)\right]\right\}$}
\STATE{$F \leftarrow F \cup \left\{\left[\bm{B},
\left(\bm{C}^\top \bm{C}\right)*\left(\bm{A}^\top\bm{A}\right),
-\mathcal{M}_{\left(2\right)} \left(\bm{C}\odot \bm{A}\right)\right]\right\}$}
\STATE{$F \leftarrow F \cup \left\{\left[\bm{C},
\left(\bm{B}^\top \bm{B}\right)*\left(\bm{A}^\top\bm{A}\right),
-\mathcal{M}_{\left(3\right)} \left(\bm{B}\odot \bm{A}\right)\right]\right\}$}
%\STATE {{\em\% Alternate the PQP updates:}}
\REPEAT
\FOR{{\bf each} $\left[\bm{X}, \bm{Q}, \bm{Z}\right]$ {\bf in} $F$}
\STATE {$\bm{\Phi} \leftarrow \lambda_{\text{min}}^{-1/2} \left(\bm{Q}\right) \sqrt{\text{diag}\left(\bm{Z}\bm{Q}^{-1}\bm{Z}^\top\right)} \text{diag}\left(\bm{Q}\right)^\top$}
\STATE { $\bm{\Phi} \leftarrow \left(\bm{\Phi} - \left|\bm{Z}\right|\right)_+ / 2 + \epsilon \bm{1}\bm{1}^\top$ }
\STATE {$\bm{X} \leftarrow \bm{X} * \left[\left(-\bm{Z}\right)_+ + \bm{\Phi}\right] \oslash \left[\bm{X}\bm{Q} + \left(\bm{Z}\right)_+ + \bm{\Phi}\right]$}
\ENDFOR
\UNTIL {$\bm{X}$ ceased to change, or reached max \#iterations}
\STATE{Normalize the columns of $\bm{A}$, $\bm{B}$, $\bm{C}$ to sum to $1$.}
\STATE{{\bf return} $\bm{A}$, $\bm{B}$, $\bm{C}$}
\end{algorithmic}
\end{algorithm}

\subsection{Proposed approach}
To summarize, the proposed approach, referred to as PTPQP, consists of three steps. Given the indexes of anchor variables $\pi^u\cup\pi^v\cup\pi^w$, the variables $\left[p\right]\backslash \left(\pi^u\cup\pi^v\cup\pi^w\right)$ are first evenly divided into $r$ partitions, and the anchor variables are added to each partition. The second step consists of forming and factorizing the sub-tensor of each partition using \cref{alg:tpqp}, this step can be parallelized. Third, normalize the anchor matrix $\left[\bm{\theta}^{\pi^u\top},\bm{\theta}^{\pi^v\top},\bm{\theta}^{\pi^w\top}\right]^\top$ formed by the anchor variable parameters to have unit column Euclidean norm, and then use either \eqref{eq:sam} or \eqref{eq:opm} to match over the anchor matrix.

\paragraph{Efficiency}
Most of the computational cost is in the factorization. Consider one partition, and let $\mathcal{M}^{\pi_j \pi_s \pi_t}$ be the corresponding sub-tensor, the sub-tensor size is $\prod\limits_{\pi\in\left\{\pi^j,\pi^s,\pi^t\right\}}\sum_{h\in\pi} d_h$. The maximum number of categories for a variable is generally a constant for the GDLM. Under smallest partitioning, this size is determined by the sub-tensor of anchor variables, i.e., $O\left(k^3\right)$, which corresponds to  $\left(p/k\right)$ partitions. One benefit of PTPQP is that the number of sub-tensor factorizations is linear in $p$ due to the partitioned factorization, this results in significant efficiency gains when $p \gg k$. Furthermore, PTPQP is easy to be parallelized across multiple CPUs and machines, since the computation as well as data are not distributed across partitions. 

\section{Provable Guarantees}
\label{sec:guarantees}
In this section, we state the main theoretical results of the proposed partitioned factorization and tensor PQP factorization.

\subsection{Sufficient conditions for guaranteed matching}
\cref{thm:pert-sam} and \cref{thm:pert-procrustes} state that when the anchor parameter matrices from two factorizations are ``close", the proposed matching algorithms obtain a consistent permutation.

\begin{theorem}
\label{thm:pert-sam}
Suppose that $\bm{\theta}_j$ is the ground-truth matrix for variable $j$. Solving \eqref{eq:sam} results in a consistent permutation if for all factors $\widehat{\bm{\theta}}_j$ of variable $j$
\begin{align*}
\frac{\left\|\bm{\theta}_{jh} - \widehat{\bm{\theta}}_{jh}\right\|_2}{\left\|\bm{\theta}_{jh}\right\|_2}
< 1 - \sqrt{\frac{1}{2} + \sqrt{\frac{1}{8}\left(1+\max_{u<v} \left(\bm{\bar{\theta}}_j^\top \bm{\bar{\theta}}_j\right)_{uv}\right)}}
 \end{align*}
for all $h\in\left[k\right]$, where $\bm{\bar{\theta}}_{jh} = \bm{\theta}_{jh} / \left\|\bm{\theta}_{jh}\right\|_2$.
\end{theorem}
\begin{proof}
Consider the smallest pair-wise angle $\alpha_{\text{min}}$ between the columns of $\bm{\bar{\theta}}_j$, we have that
\begin{align*}
\cos \alpha_{\text{min}} = \max_{u<v} \left(\bm{\bar{\theta}}_j^\top \bm{\bar{\theta}}_j\right)_{uv}.
\end{align*}
Denote by $\alpha$ the maximum angle between the column of a factorized parameter matrix $\widehat{\bm{\theta}}_j$ and the corresponding column of the ground-truth. It is sufficient to ensure that
\begin{align}
\alpha < \frac{1}{4} \alpha_{\text{min}}. \label{eq:quarter-angle}
\end{align}
Consider any two columns $s\neq t$ of the ground-truth parameter matrix, and the corresponding perturbed columns $\left\{\widehat{\bm{\theta}}_{js}, \widehat{\bm{\theta}}_{jt}\right\}$ and $\left\{\widehat{\bm{\theta}}_{js}^\prime, \widehat{\bm{\theta}}_{jt}^\prime\right\}$ from two factorizations. We have that
\begin{align*}
\angle\left(\widehat{\bm{\theta}}_{js}, \widehat{\bm{\theta}}_{js}^\prime\right)
&\leq 2\alpha\\
\angle\left(\widehat{\bm{\theta}}_{js}, \widehat{\bm{\theta}}_{jt}^\prime\right)
&\geq \angle\left(\bm{\bar{\theta}}_{js}, \bm{\bar{\theta}}_{jt}\right)
- 2\alpha\\
&\geq \alpha_{\text{min}} - 2\alpha.
\end{align*}
From \eqref{eq:quarter-angle}, we have that 
\begin{align*}
\angle\left(\widehat{\bm{\theta}}_{js}, \widehat{\bm{\theta}}_{js}^\prime\right)
<
\angle\left(\widehat{\bm{\theta}}_{js}, \widehat{\bm{\theta}}_{jt}^\prime\right),
\end{align*}
as desired for \eqref{eq:sam} to work correctly. Now consider the inner product of a perturbed column and the ground-truth, it holds that
\begin{align*}
\left<\frac{\bm{\theta}_{jh}}{\left\|\bm{\theta}_{jh}\right\|},
\frac{\bm{\theta}_{jh} + \bm{\epsilon}}{\left\|\bm{\theta}_{jh} + \bm{\epsilon}\right\|}\right> &=
\frac{\left\|\bm{\theta}_{jh}\right\|^2 - \left\|\bm{\epsilon}\right\|^2
+ \left\|\bm{\theta}_{jh} + \bm{\epsilon}\right\|^2}{2\left\|\bm{\theta}_{jh}\right\|\left\|\bm{\theta}_{jh}+\bm{\epsilon}\right\|}\\
&\geq
\frac{\left\|\bm{\theta}_{jh}\right\|- \left\|\bm{\epsilon}\right\|}{2\left\|\bm{\theta}_{jh}\right\|} +
\frac{\left\|\bm{\theta}_{jh}+\bm{\epsilon}\right\|}{2\left\|\bm{\theta}_{jh}\right\|}\\
&\geq 1 - \frac{\left\|\epsilon\right\|}{\left\|\bm{\theta}_{jh}\right\|}.
\end{align*}
Thus, a sufficient condition for \eqref{eq:sam} to yield the consistent permutation is
\begin{align*}
1 - \frac{\left\|\epsilon\right\|}{\left\|\bm{\theta}_{jh}\right\|} > \cos\left(\frac{1}{4}\alpha_{\text{min}}\right),
\end{align*}
which written in analytic form proves the theorem.
\end{proof}

\cref{thm:pert-sam} states that one obtains a consistent permutation by solving \eqref{eq:sam} in the columns of the ground-truth parameter matrix
are distinct from each other in angles and the factorized parameter matrix is near the ground-truth in Frobenius norm. Thus, a good anchor variable for the partitioned factorization \eqref{eq:part-decomp} is one whose parameter matrix has distant columns in angles.

The bound in \cref{thm:pert-sam} can be made sharp for certain $\bm{\theta}_j$, and thus the smallest angle matching algorithm has general guarantees only when the perturbation is small, i.e., the relative error ratio is less than $1- \sqrt{2+\sqrt{2}}/2 \approx 1/13$.

\begin{theorem}
\label{thm:pert-procrustes}
Suppose that $\bm{\theta}$ and $\bm{\theta}^\prime$ are two factorized parameter matrices for a variable. Solving \eqref{eq:opm} results in a consistent permutation $\psi$, if
\begin{align*}
\left\|\bm{E}\right\|_2 < \sigma_k\left(\bm{\theta}^\top \bm{\theta}\right)
\quad \text{and} \quad
-\frac{\left\|\bm{E}\right\|_2}
{\rho} \log \left(
1 - \frac{\rho}{\nu} \right)
< \frac{2-\sqrt{2}}{4}
\end{align*}
with
\begin{align*}
\rho = \sigma_1\left(\bm{E}\right) + \sigma_2\left(\bm{E}\right), \quad
\nu = \sigma_k\left(\bm{\theta}^\top \bm{\theta}\right)
+ \sigma_{k-1}\left(\bm{\theta}^\top \bm{\theta}\right)
\end{align*}
where the error matrix is define as $\bm{E} = \left(\psi\bm{\theta}\right)^\top \left( \bm{\theta}^\prime -\psi\bm{\theta}\right)$, and $\sigma_j\left(\cdot\right)$ denotes the $j$-th largest singular value.
\end{theorem}
The proof of Theorem~\ref{thm:pert-procrustes} follows from the following two Lemmas.
\begin{lemma}
\label{lem:procrustes-dom}
Suppose that $\psi$ is the consistent permutation of $\bm{\theta}$ with respect to $\bm{\theta}^\prime$. Formula \eqref{eq:opm} is guaranteed to recover $\psi$, if
\begin{align}
\text{diag}\left(\bm{U}\bm{V}^\top\right) \succ \frac{\sqrt{2}}{2} \bm{1}, \label{eq:procrustes-dom}
\end{align}
where $\bm{U}$ and $\bm{V}$ are the left and right singular matrices of $\left(\psi \bm{\theta}\right)^\top \bm{\theta}^\prime$.
\end{lemma}
\begin{proof}
We need to show that \eqref{eq:opm} yields $\psi$ for the orthogonal Procrustes problem $\min\limits_{\bm{\Psi}^\top \bm{\Psi} = \bm{I}} \left\|\bm{\theta} \bm{\Psi} - \bm{\theta}^\prime\right\|_F$. From the solution \eqref{eq:procrustes-sol}, it is easy to show that the minimizer $\bm{\Psi}^*$ of $\min\limits_{\bm{\Psi}^\top \bm{\Psi} = \bm{I}} \left\| \bm{\theta}^\prime \bm{\Psi} - \bm{\theta}\right\|_F$ and the minimizer $\bm{\Psi}^\prime$ of $\min\limits_{\bm{\Psi}^\top \bm{\Psi} = \bm{I}} \left\|\left(\psi \bm{\theta}\right) \bm{\Psi} - \bm{\theta}^\prime\right\|_F$ satisfy
\begin{align}
\bm{\Psi}^{\prime\top} = \psi \bm{\Psi}^*. \label{eq:procrustes-perm}
\end{align}
Note that $\bm{\Psi}^{*\top}$ is the desired minimizer of  $\min\limits_{\bm{\Psi}^\top \bm{\Psi} = \bm{I}} \left\| \bm{\theta} \bm{\Psi} - \bm{\theta}^\prime\right\|_F$, and thus it remains to show that \eqref{eq:opm} gives $\psi$ when applied to $\bm{\Psi}^{*\top}$, or equivalently 
\begin{align}
\arg\max_t \bm{\Psi}_{st}^* = \psi_s. \label{eq:procrustes-goal}
\end{align}

Since the row and column vectors of $\bm{\Psi}^*$ have unit Euclidean norm, the following dual statements imply each other
\begin{align}
\arg \max_t \bm{\Psi}_{ts}^\prime = j \quad \Leftrightarrow \quad \arg \max_t \bm{\Psi}_{jt}^\prime = s \qquad \forall j,s \in \left[k\right], \label{eq:procrustes-dual}
\end{align}
if condition $\eqref{eq:procrustes-dom}$ holds. Under this condition, we also have that \eqref{eq:opm} gives the identity permutation $\left[1,2,\cdots,k\right]$ for the orthogonal Procrustes problem $\min\limits_{\bm{\Psi}^\top \bm{\Psi} = \bm{I}} \left\|\left(\psi \bm{\theta}\right) \bm{\Psi} - \bm{\theta}^\prime\right\|_F$. Thus, applying \eqref{eq:opm} to both sides of \eqref{eq:procrustes-perm} yields
\begin{align*}
\arg\max_t \bm{\Psi}_{t\psi_s}^* = s,
\end{align*}
which implies \eqref{eq:procrustes-goal} from \eqref{eq:procrustes-dual}.
\end{proof}

\begin{lemma} (Mathias)
\label{lem:mathias}
Suppose that $\bm{A} \in \mathbb{R}^{n \times n}$ is nonsingular. Then for any $\bm{E} \in \mathbb{R}^{n\times n}$ with $\sigma_1\left(\bm{E}\right) < \sigma_n\left(\bm{A}\right)$ and any unitarily invariant norm $\left\|\cdot\right\|$, it holds that
\begin{align*}
\left\|\mu\left(\bm{A} + \bm{E}\right) - \mu\left(\bm{A}\right)\right\| \leq
-\frac{2\left\|\bm{E}\right\|}{\vertiii{\bm{E}}_2}
\log\left(
1 - \frac{\vertiii{\bm{E}}_2}{\sigma_n\left(\bm{A}\right) +
\sigma_{n-1}\left(\bm{A}\right)}
\right),
\end{align*}
where $\mu\left(\cdot\right)$ represents the unitary factor of the polar decomposition, and $\vertiii{\cdot}_k$ is the Ky Fan $k$-norm.
\end{lemma}
\begin{proof}[Proof of \cref{thm:pert-procrustes}]
Let $\bm{H} = \left(\psi\bm{\theta}\right)^\top \psi\bm{\theta}$, which has the same singular values as $\bm{\theta}^\top \bm{\theta}$. Denote by $\mu\left(\cdot\right)$ the unitary factor of the polar decomposition. Using the fact that $\mu\left(\bm{H}\right) = \bm{I}$, the sufficient condition of \cref{lem:procrustes-dom} is restated as
\begin{align*}
\text{diag}\left(\mu\left(\bm{H}\right) - \mu\left(\bm{H} + \bm{E}\right)\right) \prec \left(1 - \frac{\sqrt{2}}{2}\right)\bm{1}.
\end{align*}
Also note that
\begin{align*}
\max_t\left|\text{diag}\left(\mu\left(\bm{H}\right) - \mu\left(\bm{H} + \bm{E}\right)\right)_t\right| \leq
\left\|\mu\left(\bm{H}\right) - \mu\left(\bm{H} + \bm{E}\right)\right\|_2.
\end{align*}
Thus, it suffices to enforce the right term to be less than $1-\sqrt{2}/2$. From \cref{lem:mathias}, this can be achieved by letting
\begin{align*}
-\frac{2\left\|\bm{E}\right\|_2}{\vertiii{\bm{E}}_2}
\log\left(
1 - \frac{\vertiii{\bm{E}}_2}{\sigma_n\left(\bm{A}\right) +
\sigma_{n-1}\left(\bm{A}\right)}\right)
\leq
1 - \frac{\sqrt{2}}{2}.
\end{align*}
\end{proof}

The first condition in \cref{thm:pert-procrustes} requires that at least one of $\bm{\theta}$ and $\bm{\theta}^\prime$ must have full column rank. We may exchange $\bm{\theta}$ and $\bm{\theta}^\prime$ in \cref{thm:pert-procrustes} to first obtain the consistent permutation of $\bm{\theta}^\prime$ with respect to $\bm{\theta}$, $\psi$ then follows immediately.

\cref{thm:pert-procrustes} states that solving \eqref{eq:opm} recovers a consistent permutation whenever the error spectral norm is small as compared to the smallest singular value of $\bm{\theta}^\top\bm{\theta}$. This is especially useful for $\bm{\theta} \in \mathbb{R}^{d \times k}$ with the number of rows $d$ much larger than the number of columns $k$. In particular, for $\bm{\theta}$ with independent and identically distributed subgaussian entries, $\sigma_k\left(\bm{\theta}^\top \bm{\theta}\right)$ is at least of the order $\left(\sqrt{d} - \sqrt{k-1}\right)^2$ \cite{Rudelson09}.

\subsection{Convergence}
\label{sec:convergence}
The following theorem states a sufficient condition for PQP to achieve linear convergence rate. The theorem statement and proof is an adaptation of results stated in \cite{Brand11}---the proof in \cite{Brand11} overlooks a required condition on $\bm{\phi}$ and the condition
 $\bm{\gamma} \succeq \text{diag}\left(Q_{jj}\right)$ in the original proof is unnecessary.

\begin{theorem}
\label{thm:convergence}
The PQP algorithm given by \eqref{eq:pqp} monotonically decreases the objective \eqref{eq:nqp} and has linear convergence, if
\begin{align}
\bm{\gamma} \succeq \left(-\bm{Q}\right)_+ \bm{1}
\quad \text{and} \quad
\bm{\phi} \succ \frac{1}{2}
\left(
\sqrt{\frac{\bm{z}^\top\bm{Q}^{-1}\bm{z}}{\lambda_{\text{min}}\left(\bm{Q}\right)}} \text{diag}\left(\bm{Q}\right) - \left|\bm{z}\right|
\right)_+, \label{eq:conv-conds}
\end{align}
where $\lambda_{\text{min}}\left(\cdot\right)$ is the smallest eigenvalue.
\end{theorem}
\begin{proof}
First, the condition $\bm{\gamma} \geq \left(-\bm{Q}\right)_+ \bm{1}$ suffices to ensure that the updates monotonically decrease \eqref{eq:nqp} \cite{Brand11a}. Thus, it remains to show the condition on $\bm{\phi}$. Suppose that the $i$-th element of the optimum $\bm{x}^*$ is perturbed by a non-zero $\epsilon > - x_i^*$. Let $\bm{x} = \bm{x}^* + \epsilon \bm{e}_i$, and applying one update gives $\bm{x}^\prime$. Denote the $i$-th row of $\bm{Q}^+$, $\bm{Q}^-$, and $\bm{Q}$ respectively by $\bm{P}_i$, $\bm{N}_i$, and $\bm{Q}_i$, then it holds that $\bm{P}_i \bm{e}_i = Q_{ii} + \gamma_i$ and $\bm{N}_i \bm{e}_i = \gamma_i$ by definition. We now consider the ratio of errors between successive iterations:
\begin{align*}
\left|\frac{x_i^\prime - x_i^*}{x_i - x_i^*}\right| &= \left|\frac{1}{\epsilon}
\left(\frac{\bm{N}_i \left(\bm{x}^* + \epsilon\bm{e}_i\right) + z_i^-}
{\bm{P}_i\left(\bm{x}^* + \epsilon\bm{e}_i\right) + z_i^+}x_i - x_i^*\right)
\right|\\
&= \left| \frac{\bm{N}_i\bm{x}^* + \epsilon \gamma_i + z_i^-}
{\bm{P}_i\bm{x}^* + \epsilon Q_{ii} + \epsilon \gamma_i + z_i^+}
\frac{x_i^* + \epsilon}{\epsilon} - \frac{x_i^*}{\epsilon}\right|\\
&= \left| \frac{\bm{N}_i\bm{x}^* + \epsilon \gamma_i + z_i^-}
{\bm{P}_i\bm{x}^* + \epsilon Q_{ii} + \epsilon \gamma_i + z_i^+}
- \frac{x_i^*}{\epsilon} \frac{\bm{Q}_i\bm{x}^* + z_i + \epsilon Q_{ii}}
{\bm{P}_i\bm{x}^* + \epsilon Q_{ii} + \epsilon \gamma_i + z_i^+}\right|.
\end{align*}
From the KKT first-order optimality condition $x_i^* \left(\bm{Q}_i \bm{x}^* + z_i\right) = 0$, we simplify the ratio as
\begin{align}
\left|\frac{x_i^\prime - x_i^*}{x_i - x_i^*}\right| &= 
\left| \frac{\bm{N}_i\bm{x}^* + \epsilon \gamma_i + z_i^- - x_i^* Q_{ii}}
{\bm{P}_i\bm{x}^* + \epsilon Q_{ii} + \epsilon \gamma_i + z_i^+}
\right|. \label{eq:ratio}
\end{align}
Observe that the denominator is nonnegative. We also have that the denominator is greater than the numerator using the KKT optimality condition $\bm{Q}_i \bm{x}^* + z_i \geq 0$:
\begin{align*}
\bm{P}_i\bm{x}^* + \epsilon Q_{ii} + \epsilon\gamma_i + z_i^+
- \left(\bm{N}_i\bm{x}^* + \epsilon\gamma_i + z_i^- - x_i^* Q_{ii}\right)
> \bm{Q}_i \bm{x}^* + z_i \geq 0.
\end{align*}
To achieve linear convergence rate, we may enforce the ratio to be less than one. Equivalently,
\begin{align*}
\bm{P}_i\bm{x}^* + \epsilon Q_{ii} + \epsilon\gamma_i + z_i^+
+ \bm{N}_i\bm{x}^* + \epsilon\gamma_i + z_i^- - x_i^* Q_{ii} > 0.
\end{align*}
It suffices to set
\begin{align}
\phi_i >
\frac{1}{2} \left(Q_{ii} x_i^* - \left|z_i\right|\right)_+. \label{eq:phi-cond-x}
\end{align}

To get rid of $\bm{x}^*$ in \eqref{eq:phi-cond-x}, we have the following inequality
\begin{align*}
\left| \frac{1}{2} \bm{x}^{*\top}\bm{Q}\bm{x}^* + \bm{z}^\top \bm{x}^*\right| \leq \frac{1}{2}\bm{z}^\top \bm{Q}^{-1} \bm{z},
\end{align*}
where the right term is the negative of the minimum of the unconstrained problem, assuming that ${Q}$ is non-singular. If $\bm{Q}$ is singular, then $\bm{x}^*$ can be unbounded. Further simplify the inequality using KKT optimality conditions as
\begin{align*}
\left|\bm{x}^{*\top} \bm{Q} \bm{x}^* \right| &\leq z^\top \bm{Q}^{-1} \bm{z}\\
\lambda_{\text{min}}\left(\bm{Q}\right)\left\|\bm{x}^*\right\|_2^2 &\leq z^\top \bm{Q}^{-1} \bm{z}\\
\left\|\bm{x}^*\right\|_\infty &\leq \sqrt{\frac{z^\top \bm{Q}^{-1} \bm{z}}{\lambda_{\text{min}}\left(\bm{Q}\right)}}.
\end{align*}
Combining with \eqref{eq:phi-cond-x} completes the proof.
\end{proof}

\section{Results on real and simulated data}
\label{sec:exp}
We compare the proposed algorithm {\bf ptpqp} with state-of-the-art approaches including: 1) the tensor power method {\bf tpm} \cite{Anandkumar14} and matrix simultaneous diagonalization, {\bf nojd0} and {\bf nojd1} \cite{Kuleshov15}---two general tensor decomposition methods; 2) nonnegative tensor factorization {\bf hals} \cite{Kim14}; and 3) generalized method of moments {\bf meld} \cite{Zhao16}. We use the online code provided by the corresponding authors.

\subsection{Learning GDLMs on simulated data}
We adapt a simulation study from \cite{Zhao16} to compare runtime and accuracy of parameter estimation. We consider a GDLM 
where each variable takes categorical values $\left\{0,1,2,3\right\}$ and the parameters of the 
Dirichlet mixing distribution are $\{\alpha_j = 0.1\}_{j=1}^k$. We initially consider $25$ variables. The true parameters for each
hidden component $h$  are drawn from the Dirichlet distribution $\text{Dir}\left(0.5,0.5,0.5,0.5\right)$. The resulting moment estimator is a $100$-by-$100$-by-$100$ tensor. We vary the number of components $k$ and add noise by replacing a fraction $\delta$ of the observations with draws from a discrete uniform distribution. We also vary the number of samples $n=100,500,1000,5000$, number of clusters  $k=3,5,10,20$, and contamination $\delta=0,0.05,0.1$. Across these settings we found that the empirical third-order estimator typically exhibits between $20\%$ and $50\%$ negative entries.

\paragraph{Accuracy of inference}
Accuracy is measured by root-mean-square error (RMSE) which we compare across algorithms as a function of the number of components 
for various sample sizes and levels of contamination, see \cref{fig:synth}. Both {\bf hals} and {\bf ptpqp} are consistently among the top estimators, and {\bf ptpqp} outperforms {\bf hals} as $n$ grows. For small sample sizes and many hidden components {\bf meld} achieves the smallest RMSE.  The RMSE of {\bf tpm} is relatively large,  probably due to the whitening technique used to approximately transform the nonorthogonal factorization into an orthogonal one, see \cite{Souloumiac09,Colombo16}. The most relevant observation is that  {\bf ptpqp} outperforms other methods for large, noisy data. 

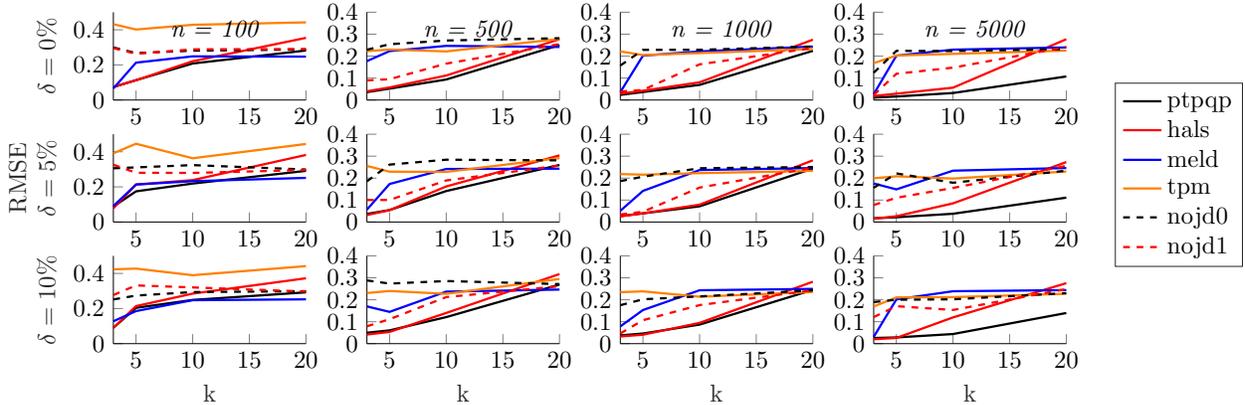
\begin{figure*}[tb]
\centering
\begin{adjustbox}{width=\textwidth}
\input{figure/synthetic}
\end{adjustbox}
%\vskip -0.1in
\caption{RMSE between inferred parameters and the ground truth.}
\label{fig:synth}
\vskip -0.1in
\end{figure*}

\paragraph{Computational cost} We examined how runtime scales as a function of the number of partitions. For the same model we set $p=1000$ variables and $n=1000$ samples. The tensor is now $4000$-by-$4000$-by-$4000$. We evaluated the runtime of {\bf ptpqp} (without parallelization) with the 
number of partitions set to $\left\{30,40,50,100,200\right\}$. On a laptop with Intel i7-4702HQ@2.20GHz CPU and 8GB memory, {\bf ptpqp} with $100$ partitions completes within $3.5$ min, $4$ min, and $5$ min for $k=4,8,12$, respectively. In addition, the runtime monotonically decreases with the number of partitions. Further speedups can be obtained by parallelizing the factorization of partitions across multiple CPUs or machines.

\subsection{Predicting crowdsourced labels}
In \cite{Zhang14}, a combination of EM and tensor decompositions was used to predict crowdsourcing annotations. The task is to predict the true label given incomplete and noisy observations from a set of workers, this is a mixed membership problem \cite{Dawid}. In \cite{Zhang14} a third-order tensor estimator was proposed to obtain an initial estimate for the EM algorithm. We compare the predictive performance on five data sets of several tensor decomposition methods as well as the EM algorithm initialized with majority voting by the workers ({\bf MV+EM}). The fraction of incorrect predictions and the size of each dataset are in the table below. Note that {\bf ptpqp} matches or outperforms the other tensor methods on all but one dataset, and even outperforms {\bf MV+EM} on two datasets.

\begin{table}[htb]
\vskip 0.15in
\caption{Incorrectly predicted labels (\%)}
\vskip -0.15in
\label{tbl:crowd}
\begin{center}
\begin{small}
\begin{sc}
\begin{tabular}{c|ccccc}
Dataset  & Birds   & RTE    & TREC   & Dogs   & Web \\ \hline
ptpqp &{\bf 11.11} & 7.75   & 30.81  &{\bf15.37}&{\bf 14.44}  \\
hals     & 12.96   & 7.75   & 31.47  & 20.57  & 26.84 \\ 
tpm   &{\bf 11.11} & 7.62   & 31.87  & 15.49  & 14.70 \\
nojd0    & 12.04   & 8.00   & 32.97  & 15.49  & 18.39 \\
nojd1    & 12.04   & 8.00   & 35.91  & 15.86  & 25.97 \\ \hline
MV+EM &{\bf11.11}&{\bf7.12}&{\bf30.20}  & 15.86  & 15.91 \\ \hline
Size     & 108     & 800    & 19033  & 807    & 2665
\end{tabular}
\end{sc}
\end{small}
\end{center}
\vskip -0.15in
\end{table}

\section{Conclusions}
We proposed an efficient algorithm for learning mixed mixture models based on the idea of partitioned factorizations. The key challenge is to consistently match the partitioned parameters with the hidden components. We provided sufficient conditions to ensure consistency. In addition, we have also developed a nonnegative approximation to handle the negative entries in the empirical method of moments estimators, a problem not addressed by several recent tensor methods. Results on synthetic and real data corroborate that the proposed approach achieves improved inference accuracy as well as computational efficiency than state-of-the-art methods.

\section*{Code}

Code for all the simulations is available from Zilong Tan's GitHub repository \\
\url{https://github.com/ZilongTan/ptpqp}. 

\section*{Acknowledgements}
Z.T.\ would like to thank Rong Ge for sharing helpful insights. S.M.\  would like to thank Lek-Heng Lim for insights. Z.T.\ would like to acknowledge the support of grants NSF CNS-1423128,  NSF IIS-1423124, and NSF CNS-1218981. S.M.\ would like to acknowledge the support of grants NSF IIS-1546331, NSF DMS-1418261, NSF IIS-1320357, NSF DMS-1045153, and NSF DMS-1613261. 

\bibliography{ref}
\bibliographystyle{IEEE}

\newpage
\appendix
\section{Dirichlet Moments}
\label{sec:moments}
For a Dirichlet random vector $\bm{x}$ with concentration parameters $\bm{\alpha}$, the component moments can be easily shown by the integral
\begin{align*}
\mathbb{E}\left[x_i\right] &= \frac{\alpha_i}{\alpha_0}
\quad
&\mathbb{E}\left[x_i x_j\right] &=
\frac{\alpha_i \alpha_j}{\alpha_0\left(\alpha_0 + 1\right)}, \quad i\neq j\\
\mathbb{E}\left[x_i^2\right] &=\frac{\alpha_i\left(\alpha_i+1\right)}{\alpha_0\left(\alpha_0+1\right)}
\quad
&\mathbb{E}\left[x_j x_s x_t\right] &= \frac{\alpha_j \alpha_s \alpha_t}{\alpha_0\left(\alpha_0+1\right)\left(\alpha_0+2\right)}, \quad j\neq s \neq t\\
\mathbb{E}\left[x_j^2 x_s\right] &= \frac{\alpha_j\left(\alpha_j + 1\right) \alpha_s}{\alpha_0\left(\alpha_0+1\right)\left(\alpha_0+2\right)}, \quad j\neq s
\quad
&\mathbb{E}\left[x_j^3\right] &= \frac{\alpha_j\left(\alpha_j + 1\right) \left(\alpha_j + 2\right)}{\alpha_0\left(\alpha_0+1\right)\left(\alpha_0+2\right)},
\end{align*}
where $\alpha_0 = \sum_{i\geq 1} \alpha_i$. Comparing the second and third order component moments, we arrive at the following cross-moments:
\begin{align}
\mathbb{E}\left[\bm{x}\right] &= \frac{1}{\alpha_0} \bm{\alpha}\\
\mathbb{E}\left[\bm{x} \times \bm{x}\right] &= \frac{1}{\alpha_0\left(\alpha_0 + 1\right)} \bm{\alpha} \bm{\alpha}^\top + \frac{1}{\alpha_0\left(\alpha_0 + 1\right)} \text{diag}\left(\bm{\alpha}\right) = \frac{\alpha_0}{\alpha_0+1} \mathbb{E}\left[\bm{x}\right] \mathbb{E}\left[\bm{x}\right]^\top + \frac{\text{diag}\left(\bm{\alpha}\right)}{\alpha_0\left(\alpha_0+1\right)} \label{eq:2nd-moment}\\
\mathbb{E}\left[\bm{x} \times \bm{x} \times \bm{x}\right] &= 
\frac{1}{\alpha_0 \left(\alpha_0 + 1\right) \left(\alpha_0 + 2\right) } 
\left[
\bm{\alpha} \times \bm{\alpha} \times \bm{\alpha} +
\sum_{i\geq 1} \alpha_i \left(\bm{\alpha} \times \bm{e}_i \times \bm{e}_i\right) +
\sum_{i\geq 1} \alpha_i \left(\bm{e}_i \times \bm{\alpha} \times \bm{e}_i \right) \right. \notag\\
&\phantom{=} +\left.
\sum_{i\geq 1} \alpha_i \left(\bm{e}_i \times \bm{e}_i \times \bm{\alpha} \right) + 2\sum_{i\geq 1} \alpha_i
\left(\bm{e}_i \times \bm{e}_i \times \bm{e}_i\right)
\right]. \label{eq:3rd-moment}
\end{align}
To express the parameters as third-order cross-moments, first observe that the following holds for a Dirichlet random vector $\bm{x}$:
\begin{dmath*}
\frac{1}{\alpha_0^2 \left(\alpha_0 + 1\right)} \sum_{i\geq 1} \alpha_i \left(\bm{\alpha}\times\bm{e}_i\times\bm{e}_i + \bm{e}_i\times\bm{\alpha}\times\bm{e}_i + \bm{e}_i\times\bm{e}_i\times\bm{\alpha}\right)
= \mathbb{E}\left[\mathbb{E}\left[\bm{x}\right] \times \bm{x} \times \bm{x}\right]
+ \mathbb{E}\left[\bm{x} \times \mathbb{E}\left[\bm{x}\right] \times \bm{x}\right]
+ \mathbb{E}\left[\bm{x} \times \bm{x} \times \mathbb{E}\left[\bm{x}\right]\right] - \frac{3\alpha_0}{\alpha_0 + 1} \mathbb{E}\left[\bm{x}\right] \times \mathbb{E}\left[\bm{x}\right] \times \mathbb{E}\left[\bm{x}\right].
\end{dmath*}
This is an immediate result from the the second-order component moments. Combining with \eqref{eq:3rd-moment} yields
\begin{dmath}
\sum_{i\geq 1} \frac{2\alpha_i \bm{e}_i \times \bm{e}_i \times \bm{e}_i}{\alpha_0 \left(\alpha_0+1\right) \left(\alpha_0+2\right) }  = \mathbb{E}\left[\bm{x} \times \bm{x} \times \bm{x} \right]
+ \frac{2\alpha_0^2}{\left(\alpha_0 + 1\right)\left(\alpha_0 + 2\right)}  \mathbb{E}\left[\bm{x}\right] \times \mathbb{E}\left[\bm{x}\right] \times \mathbb{E}\left[\bm{x}\right] 
- \frac{\alpha_0}{\alpha_0 + 2} \left(
 \mathbb{E}\left[\mathbb{E}\left[\bm{x}\right] \times \bm{x} \times \bm{x}\right]
+ \mathbb{E}\left[\bm{x} \times \mathbb{E}\left[\bm{x}\right] \times \bm{x}\right]
+ \mathbb{E}\left[\bm{x} \times \bm{x} \times \mathbb{E}\left[\bm{x}\right]\right]
\right). \label{eq:3rd-moment-final}
\end{dmath}

\subsection{Derivation of moment estimators}
Our goal is to derive the estimators of parameter vectors $\bm{\theta}_j$ for each variable $j$ using the first- and second- order empirical cross-moments of $\bm{b}_{ij}$. In GDLM, the expectation of variable $j$ conditioned on $\bm{x}$ is written
\begin{align*}
\mathbb{E}\left[\bm{b}_{ij} \vert \; \bm{x}\right] = \bm{\theta}_j \bm{x}.
\end{align*}
Thus, the expected observation of variable $j$ is given by
\begin{align}
\mathbb{E}\left[\bm{b}_{ij}\right] = \mathbb{E}\left[\mathbb{E}\left[\bm{b}_{ij} \vert \; \bm{x}\right]\right] = \bm{\theta}_j \mathbb{E}\left[\bm{x}\right] = \frac{\bm{\theta}_j \bm{\alpha}}{\alpha_0} \label{eq:exp_y}.
\end{align}
Now consider two variables $\bm{b}_{ij}$ and $\bm{b}_{is}$ which are generated with the same latent factors $\bm{x}$. Combining \eqref{eq:exp_y} and \eqref{eq:2nd-moment} to obtain
\begin{align*}
\sum_{r\geq 1} \frac{\alpha_r}{\alpha_0 \left(\alpha_0 + 1\right)} \bm{\theta}_{jr} \times \bm{\theta}_{sr} 
=
\mathbb{E}\left[\bm{b}_{ij} \times \bm{b}_{is}\right] - \frac{\alpha_0}{\alpha_0 + 1} \mathbb{E}\left[\bm{b}_{ij}\right] \mathbb{E}\left[\bm{b}_{is}\right]^\top.
\end{align*}
For three variables $\bm{b}_{ij}$, $\bm{b}_{is}$, and $\bm{b}_{it}$, we can write  $\mathbb{E}\left[\bm{b}_{ij} \times \bm{b}_{is} \times \bm{b}_{it}\right] = \mathbb{E}\left[\bm{x} \times \bm{x} \times \bm{x}\right] \times_1 \bm{\theta}_j \times_2 \bm{\theta}_s \times_3 \bm{\theta}_t$. Using \eqref{eq:3rd-moment-final}, we establish that
\begin{dmath}
\sum_{i\geq 1} \frac{2\alpha_i \bm{\theta}_j \times \bm{\theta}_s \times \bm{\theta}_t}{\alpha_0 \left(\alpha_0+1\right) \left(\alpha_0+2\right) }  = \mathbb{E}\left[\bm{b}_{ij} \times \bm{b}_{is} \times \bm{b}_{it} \right]
+ \frac{2\alpha_0^2}{\left(\alpha_0 + 1\right)\left(\alpha_0 + 2\right)}  \mathbb{E}\left[\bm{b}_{ij}\right] \times \mathbb{E}\left[\bm{b}_{is}\right] \times \mathbb{E}\left[\bm{b}_{it}\right] 
- \frac{\alpha_0}{\alpha_0 + 2} \left(
 \mathbb{E}\left[\mathbb{E}\left[\bm{b}_{ij}\right] \times \bm{b}_{is} \times \bm{b}_{it}\right]
+ \mathbb{E}\left[\bm{b}_{ij} \times \mathbb{E}\left[\bm{b}_{is}\right] \times \bm{b}_{it}\right]
+ \mathbb{E}\left[\bm{b}_{ij} \times \bm{b}_{is} \times \mathbb{E}\left[\bm{b}_{it}\right]\right]
\right).
\end{dmath}

\section{Approximate Orthogonalization in the Tensor Power Method}
\label{sec:tpm_preproc}
TPM requires the tensor to be decomposed to be symmetric, and the factor matrices to be orthogonal. Specifically, it performs the following decomposition
\begin{align}
\mathcal{M}_3^\prime = \sum_{i=1}^r \lambda_i \bm{u}_i \times \bm{u}_i \times \bm{u}_i, \label{eq:tod}
\end{align}
where $\bm{u}_i$ are orthonormal vectors. Thus, TPM does not immediately apply to the general CP decomposition \eqref{eq:cp}.

The general resolution is to first use the symmetric tensor embedding \cite{Ragnarsson13,Anandkumar16}, forming a larger symmetric tensor $\mathcal{M}_3$ that contains the asymmetric tensor to be decomposed. The formed $\mathcal{M}_3$ is a sparse $\left(\sum_{i=1}^p d_i\right)$-by-$\left(\sum_{i=1}^p d_i\right)$-by-$\left(\sum_{i=1}^p d_i\right)$ tensor of which $7/9$ entries are zero. The space and computation complexities rapidly become prohibitive when the number of variables $p$ and the category counts $d_j$ grow. 

Next, TPM requires an addition empirical second-order estimator $\widehat{\mathcal{M}}_2$ for orthogonalizing the factor matrices of $\mathcal{M}_3$ to obtain $\mathcal{M}_3^\prime$ \cite{Anandkumar14}. This is done by computing the whitening transformation from $\widehat{\mathcal{M}}_2$. However, the whitening technique based on empirical $\widehat{\mathcal{M}}_2$ is often a cause of suboptimal performance \cite{Souloumiac09,Colombo16}.

\end{document}

%% file: preamble.tex
\usepackage[utf8]{inputenc}
\usepackage{dsfont}
\usepackage{booktabs}
\usepackage{nicefrac}
\usepackage{microtype}

\usepackage[english]{babel}

\usepackage{graphicx}
\usepackage{pgfplots}
\pgfplotsset{compat=newest}

\usetikzlibrary{plotmarks}
\usetikzlibrary{arrows.meta}
\usepgfplotslibrary{patchplots}
\usepackage{grffile}

\usepackage{bm}
\usepackage{fullpage}
\usepackage[small]{caption}
\usepackage{float}
\usepackage{balance}

\usepackage{amsmath}
\usepackage{amsfonts}
\usepackage{amssymb}
\usepackage{amsthm}
\usepackage{algorithm,algorithmic}
\usepackage{breqn}

\usepackage{fancyhdr}

\newtheorem{theorem}{Theorem}
\newtheorem{lemma}{Lemma}

\definecolor{DarkGreen}{rgb}{0.1,0.5,0.1}
\definecolor{DarkRed}{rgb}{0.5,0.1,0.1}
\definecolor{DarkBlue}{rgb}{0.1,0.1,0.5}

\usepackage[backref,colorlinks,bookmarks=true]{hyperref}
\hypersetup{
    citecolor=DarkGreen,  % color of links to bibliography
    urlcolor=DarkBlue,    % color of external links
}

\usepackage{adjustbox}
\usepackage[capitalize,nameinlink,noabbrev]{cleveref}
\crefname{section}{\S}{\S\S}
\Crefname{section}{\S}{\S\S}

\DeclareMathOperator{\tr}{tr}

\newcommand{\vertiii}[1]{{\left\vert\kern-0.25ex\left\vert\kern-0.25ex\left\vert #1\right\vert\kern-0.25ex\right\vert\kern-0.25ex\right\vert}}

\makeatletter
\renewcommand*\env@matrix[1][*\c@MaxMatrixCols c]{%
  \hskip -\arraycolsep
  \let\@ifnextchar\new@ifnextchar
  \array{#1}}
\makeatother

%% file: figure/synthetic.tex
\begin{tikzpicture}

\begin{axis}[%
width=6.484in,
height=2.247in,
at={(0.5in,0.176in)},
scale only axis,
xmin=0,
xmax=1,
ymin=0,
ymax=1,
ylabel style={font=\color{white!15!black}},
ylabel={RMSE},
axis line style={draw=none},
ticks=none,
axis x line*=bottom,
axis y line*=left
]
\end{axis}

\begin{axis}[%
width=1.188in,
height=0.542in,
at={(5.675in,0.276in)},
scale only axis,
xmin=3,
xmax=20,
xlabel style={font=\color{white!15!black}},
xlabel={k},
ymin=0,
ymax=0.4,
axis background/.style={fill=white},
axis x line*=bottom,
axis y line*=left
]
\addplot [color=black, line width=1.0pt, forget plot]
  table[row sep=crcr]{%
3	0.025728\\
5	0.028376\\
10	0.043854\\
20	0.139676\\
};
\addplot [color=red, line width=1.0pt, forget plot]
  table[row sep=crcr]{%
3	0.020296\\
5	0.025965\\
10	0.119614\\
20	0.275731\\
};
\addplot [color=blue, line width=1.0pt, forget plot]
  table[row sep=crcr]{%
3	0.030989\\
5	0.201581\\
10	0.238855\\
20	0.244042\\
};
\addplot [color=orange, line width=1.0pt, forget plot]
  table[row sep=crcr]{%
3	0.169321\\
5	0.210687\\
10	0.211861\\
20	0.227154\\
};
\addplot [color=black, dashed, line width=1.0pt, forget plot]
  table[row sep=crcr]{%
3	0.189587\\
5	0.200141\\
10	0.20204\\
20	0.230185\\
};
\addplot [color=red, dashed, line width=1.0pt, forget plot]
  table[row sep=crcr]{%
3	0.121096\\
5	0.170784\\
10	0.153482\\
20	0.248628\\
};
\end{axis}

\begin{axis}[%
width=1.188in,
height=0.542in,
at={(5.675in,1.028in)},
scale only axis,
xmin=3,
xmax=20,
ymin=0,
ymax=0.4,
axis background/.style={fill=white},
axis x line*=bottom,
axis y line*=left
]
\addplot [color=black, line width=1.0pt, forget plot]
  table[row sep=crcr]{%
3	0.017214\\
5	0.021688\\
10	0.037592\\
20	0.110992\\
};
\addplot [color=red, line width=1.0pt, forget plot]
  table[row sep=crcr]{%
3	0.014799\\
5	0.02639\\
10	0.084883\\
20	0.273098\\
};
\addplot [color=blue, line width=1.0pt, forget plot]
  table[row sep=crcr]{%
3	0.17546\\
5	0.148608\\
10	0.233876\\
20	0.245928\\
};
\addplot [color=orange, line width=1.0pt, forget plot]
  table[row sep=crcr]{%
3	0.199108\\
5	0.208232\\
10	0.197862\\
20	0.230024\\
};
\addplot [color=black, dashed, line width=1.0pt, forget plot]
  table[row sep=crcr]{%
3	0.155756\\
5	0.221409\\
10	0.178861\\
20	0.233672\\
};
\addplot [color=red, dashed, line width=1.0pt, forget plot]
  table[row sep=crcr]{%
3	0.077307\\
5	0.11037\\
10	0.155119\\
20	0.255816\\
};
\end{axis}

\begin{axis}[%
width=1.188in,
height=0.542in,
at={(5.675in,1.781in)},
scale only axis,
clip=false,
xmin=3,
xmax=20,
ymin=0,
ymax=0.4,
axis background/.style={fill=white},
axis x line*=bottom,
axis y line*=left
]
\addplot [color=black, line width=1.0pt, forget plot]
  table[row sep=crcr]{%
3	0.010957\\
5	0.016074\\
10	0.031925\\
20	0.107308\\
};
\addplot [color=red, line width=1.0pt, forget plot]
  table[row sep=crcr]{%
3	0.018271\\
5	0.028897\\
10	0.056237\\
20	0.277002\\
};
\addplot [color=blue, line width=1.0pt, forget plot]
  table[row sep=crcr]{%
3	0.02653\\
5	0.202026\\
10	0.2295\\
20	0.239625\\
};
\addplot [color=orange, line width=1.0pt, forget plot]
  table[row sep=crcr]{%
3	0.16792\\
5	0.202699\\
10	0.210184\\
20	0.225235\\
};
\addplot [color=black, dashed, line width=1.0pt, forget plot]
  table[row sep=crcr]{%
3	0.124101\\
5	0.223848\\
10	0.223074\\
20	0.231821\\
};
\addplot [color=red, dashed, line width=1.0pt, forget plot]
  table[row sep=crcr]{%
3	0.022717\\
5	0.119507\\
10	0.147479\\
20	0.241625\\
};
\node[below, align=center]
at (rel axis cs:0.5,1) {${\text{\it{} n = 5000}}$};
\end{axis}

\begin{axis}[%
width=1.188in,
height=0.542in,
at={(4.112in,0.276in)},
scale only axis,
xmin=3,
xmax=20,
xlabel style={font=\color{white!15!black}},
xlabel={k},
ymin=0,
ymax=0.4,
axis background/.style={fill=white},
axis x line*=bottom,
axis y line*=left
]
\addplot [color=black, line width=1.0pt, forget plot]
  table[row sep=crcr]{%
3	0.037808\\
5	0.046178\\
10	0.086905\\
20	0.245006\\
};
\addplot [color=red, line width=1.0pt, forget plot]
  table[row sep=crcr]{%
3	0.032616\\
5	0.040788\\
10	0.094723\\
20	0.282324\\
};
\addplot [color=blue, line width=1.0pt, forget plot]
  table[row sep=crcr]{%
3	0.078552\\
5	0.153241\\
10	0.243263\\
20	0.249475\\
};
\addplot [color=orange, line width=1.0pt, forget plot]
  table[row sep=crcr]{%
3	0.234575\\
5	0.238295\\
10	0.213843\\
20	0.235313\\
};
\addplot [color=black, dashed, line width=1.0pt, forget plot]
  table[row sep=crcr]{%
3	0.175222\\
5	0.202125\\
10	0.215261\\
20	0.242622\\
};
\addplot [color=red, dashed, line width=1.0pt, forget plot]
  table[row sep=crcr]{%
3	0.046443\\
5	0.106833\\
10	0.175709\\
20	0.24313\\
};
\end{axis}

\begin{axis}[%
width=1.188in,
height=0.542in,
at={(4.112in,1.028in)},
scale only axis,
xmin=3,
xmax=20,
ymin=0,
ymax=0.4,
axis background/.style={fill=white},
axis x line*=bottom,
axis y line*=left
]
\addplot [color=black, line width=1.0pt, forget plot]
  table[row sep=crcr]{%
3	0.027154\\
5	0.039352\\
10	0.071452\\
20	0.243719\\
};
\addplot [color=red, line width=1.0pt, forget plot]
  table[row sep=crcr]{%
3	0.025406\\
5	0.039679\\
10	0.079037\\
20	0.281133\\
};
\addplot [color=blue, line width=1.0pt, forget plot]
  table[row sep=crcr]{%
3	0.050972\\
5	0.141752\\
10	0.236661\\
20	0.245811\\
};
\addplot [color=orange, line width=1.0pt, forget plot]
  table[row sep=crcr]{%
3	0.218452\\
5	0.215609\\
10	0.222549\\
20	0.232875\\
};
\addplot [color=black, dashed, line width=1.0pt, forget plot]
  table[row sep=crcr]{%
3	0.186303\\
5	0.206593\\
10	0.244901\\
20	0.250548\\
};
\addplot [color=red, dashed, line width=1.0pt, forget plot]
  table[row sep=crcr]{%
3	0.034438\\
5	0.046679\\
10	0.156865\\
20	0.247183\\
};
\end{axis}

\begin{axis}[%
width=1.188in,
height=0.542in,
at={(4.112in,1.781in)},
scale only axis,
clip=false,
xmin=3,
xmax=20,
ymin=0,
ymax=0.4,
axis background/.style={fill=white},
axis x line*=bottom,
axis y line*=left
]
\addplot [color=black, line width=1.0pt, forget plot]
  table[row sep=crcr]{%
3	0.023412\\
5	0.036506\\
10	0.068975\\
20	0.224058\\
};
\addplot [color=red, line width=1.0pt, forget plot]
  table[row sep=crcr]{%
3	0.028998\\
5	0.042295\\
10	0.080726\\
20	0.27503\\
};
\addplot [color=blue, line width=1.0pt, forget plot]
  table[row sep=crcr]{%
3	0.035772\\
5	0.202238\\
10	0.22147\\
20	0.243505\\
};
\addplot [color=orange, line width=1.0pt, forget plot]
  table[row sep=crcr]{%
3	0.221521\\
5	0.204841\\
10	0.213863\\
20	0.232196\\
};
\addplot [color=black, dashed, line width=1.0pt, forget plot]
  table[row sep=crcr]{%
3	0.157001\\
5	0.228469\\
10	0.22877\\
20	0.243986\\
};
\addplot [color=red, dashed, line width=1.0pt, forget plot]
  table[row sep=crcr]{%
3	0.038541\\
5	0.04587\\
10	0.162595\\
20	0.237083\\
};
\node[below, align=center]
at (rel axis cs:0.5,1) {${\text{\it{} n = 1000}}$};
\end{axis}

\begin{axis}[%
width=1.188in,
height=0.542in,
at={(2.549in,0.276in)},
scale only axis,
xmin=3,
xmax=20,
xlabel style={font=\color{white!15!black}},
xlabel={k},
ymin=0,
ymax=0.4,
axis background/.style={fill=white},
axis x line*=bottom,
axis y line*=left
]
\addplot [color=black, line width=1.0pt, forget plot]
  table[row sep=crcr]{%
3	0.049777\\
5	0.060491\\
10	0.120984\\
20	0.26914\\
};
\addplot [color=red, line width=1.0pt, forget plot]
  table[row sep=crcr]{%
3	0.041684\\
5	0.053333\\
10	0.139856\\
20	0.317217\\
};
\addplot [color=blue, line width=1.0pt, forget plot]
  table[row sep=crcr]{%
3	0.169701\\
5	0.144842\\
10	0.237109\\
20	0.246445\\
};
\addplot [color=orange, line width=1.0pt, forget plot]
  table[row sep=crcr]{%
3	0.230612\\
5	0.240018\\
10	0.226985\\
20	0.294757\\
};
\addplot [color=black, dashed, line width=1.0pt, forget plot]
  table[row sep=crcr]{%
3	0.288095\\
5	0.273704\\
10	0.285261\\
20	0.271434\\
};
\addplot [color=red, dashed, line width=1.0pt, forget plot]
  table[row sep=crcr]{%
3	0.079175\\
5	0.111423\\
10	0.212127\\
20	0.261741\\
};
\end{axis}

\begin{axis}[%
width=1.188in,
height=0.542in,
at={(2.549in,1.028in)},
scale only axis,
xmin=3,
xmax=20,
ymin=0,
ymax=0.4,
axis background/.style={fill=white},
axis x line*=bottom,
axis y line*=left
]
\addplot [color=black, line width=1.0pt, forget plot]
  table[row sep=crcr]{%
3	0.036692\\
5	0.054438\\
10	0.141062\\
20	0.259847\\
};
\addplot [color=red, line width=1.0pt, forget plot]
  table[row sep=crcr]{%
3	0.031707\\
5	0.054003\\
10	0.161978\\
20	0.303395\\
};
\addplot [color=blue, line width=1.0pt, forget plot]
  table[row sep=crcr]{%
3	0.056219\\
5	0.173163\\
10	0.241534\\
20	0.242869\\
};
\addplot [color=orange, line width=1.0pt, forget plot]
  table[row sep=crcr]{%
3	0.255369\\
5	0.229134\\
10	0.228335\\
20	0.291274\\
};
\addplot [color=black, dashed, line width=1.0pt, forget plot]
  table[row sep=crcr]{%
3	0.185247\\
5	0.261493\\
10	0.283821\\
20	0.28082\\
};
\addplot [color=red, dashed, line width=1.0pt, forget plot]
  table[row sep=crcr]{%
3	0.100169\\
5	0.10042\\
10	0.189273\\
20	0.257357\\
};
\end{axis}

\begin{axis}[%
width=1.188in,
height=0.542in,
at={(2.549in,1.781in)},
scale only axis,
clip=false,
xmin=3,
xmax=20,
ymin=0,
ymax=0.4,
axis background/.style={fill=white},
axis x line*=bottom,
axis y line*=left
]
\addplot [color=black, line width=1.0pt, forget plot]
  table[row sep=crcr]{%
3	0.035827\\
5	0.051872\\
10	0.093486\\
20	0.250223\\
};
\addplot [color=red, line width=1.0pt, forget plot]
  table[row sep=crcr]{%
3	0.039877\\
5	0.057141\\
10	0.111605\\
20	0.27706\\
};
\addplot [color=blue, line width=1.0pt, forget plot]
  table[row sep=crcr]{%
3	0.177101\\
5	0.222808\\
10	0.246668\\
20	0.241669\\
};
\addplot [color=orange, line width=1.0pt, forget plot]
  table[row sep=crcr]{%
3	0.222782\\
5	0.229866\\
10	0.221376\\
20	0.27824\\
};
\addplot [color=black, dashed, line width=1.0pt, forget plot]
  table[row sep=crcr]{%
3	0.229097\\
5	0.254173\\
10	0.271143\\
20	0.281723\\
};
\addplot [color=red, dashed, line width=1.0pt, forget plot]
  table[row sep=crcr]{%
3	0.089502\\
5	0.094862\\
10	0.16588\\
20	0.254643\\
};
\node[below, align=center]
at (rel axis cs:0.5,1) {${\text{\it{} n = 500}}$};
\end{axis}

\begin{axis}[%
width=1.188in,
height=0.542in,
at={(0.986in,0.276in)},
scale only axis,
xmin=3,
xmax=20,
xlabel style={font=\color{white!15!black}},
xlabel={k},
ymin=0,
ymax=0.5,
ylabel style={font=\color{white!15!black}},
ylabel={$\delta = 10\%$},
axis background/.style={fill=white},
axis x line*=bottom,
axis y line*=left
]
\addplot [color=black, line width=1.0pt, forget plot]
  table[row sep=crcr]{%
3	0.090501\\
5	0.204159\\
10	0.248942\\
20	0.292302\\
};
\addplot [color=red, line width=1.0pt, forget plot]
  table[row sep=crcr]{%
3	0.090276\\
5	0.213942\\
10	0.284979\\
20	0.372407\\
};
\addplot [color=blue, line width=1.0pt, forget plot]
  table[row sep=crcr]{%
3	0.127595\\
5	0.185457\\
10	0.247612\\
20	0.252456\\
};
\addplot [color=orange, line width=1.0pt, forget plot]
  table[row sep=crcr]{%
3	0.423331\\
5	0.427455\\
10	0.389795\\
20	0.44162\\
};
\addplot [color=black, dashed, line width=1.0pt, forget plot]
  table[row sep=crcr]{%
3	0.251873\\
5	0.273853\\
10	0.293865\\
20	0.298854\\
};
\addplot [color=red, dashed, line width=1.0pt, forget plot]
  table[row sep=crcr]{%
3	0.277026\\
5	0.331708\\
10	0.320305\\
20	0.296922\\
};
\end{axis}

\begin{axis}[%
width=1.188in,
height=0.542in,
at={(0.986in,1.028in)},
scale only axis,
xmin=3,
xmax=20,
ymin=0,
ymax=0.5,
ylabel style={font=\color{white!15!black}},
ylabel={$\delta= 5\%$},
axis background/.style={fill=white},
axis x line*=bottom,
axis y line*=left
]
\addplot [color=black, line width=1.0pt, forget plot]
  table[row sep=crcr]{%
3	0.087798\\
5	0.174665\\
10	0.21895\\
20	0.291283\\
};
\addplot [color=red, line width=1.0pt, forget plot]
  table[row sep=crcr]{%
3	0.077726\\
5	0.213578\\
10	0.238473\\
20	0.381885\\
};
\addplot [color=blue, line width=1.0pt, forget plot]
  table[row sep=crcr]{%
3	0.090937\\
5	0.213365\\
10	0.232543\\
20	0.251249\\
};
\addplot [color=orange, line width=1.0pt, forget plot]
  table[row sep=crcr]{%
3	0.391364\\
5	0.4458\\
10	0.362918\\
20	0.444456\\
};
\addplot [color=black, dashed, line width=1.0pt, forget plot]
  table[row sep=crcr]{%
3	0.306056\\
5	0.311027\\
10	0.324027\\
20	0.298993\\
};
\addplot [color=red, dashed, line width=1.0pt, forget plot]
  table[row sep=crcr]{%
3	0.32641\\
5	0.280314\\
10	0.279601\\
20	0.295993\\
};
\end{axis}

\begin{axis}[%
width=1.188in,
height=0.542in,
at={(0.986in,1.781in)},
scale only axis,
clip=false,
xmin=3,
xmax=20,
ymin=0,
ymax=0.5,
ylabel style={font=\color{white!15!black}},
ylabel={$\delta = 0\%$},
axis background/.style={fill=white},
axis x line*=bottom,
axis y line*=left,
legend style={at={(5.2,-1.9)}, anchor=south west, legend cell align=left, align=left, draw=white!15!black}
]
\addplot [color=black, line width=1.0pt]
  table[row sep=crcr]{%
3	0.073608\\
5	0.113211\\
10	0.208036\\
20	0.282168\\
};
\addlegendentry{ptpqp}

\addplot [color=red, line width=1.0pt]
  table[row sep=crcr]{%
3	0.07522\\
5	0.112413\\
10	0.220901\\
20	0.35403\\
};
\addlegendentry{hals}

\addplot [color=blue, line width=1.0pt]
  table[row sep=crcr]{%
3	0.066035\\
5	0.212891\\
10	0.24867\\
20	0.247158\\
};
\addlegendentry{meld}

\addplot [color=orange, line width=1.0pt]
  table[row sep=crcr]{%
3	0.431798\\
5	0.401616\\
10	0.428362\\
20	0.44218\\
};
\addlegendentry{tpm}

\addplot [color=black, dashed, line width=1.0pt]
  table[row sep=crcr]{%
3	0.29956\\
5	0.267326\\
10	0.281102\\
20	0.287952\\
};
\addlegendentry{nojd0}

\addplot [color=red, dashed, line width=1.0pt]
  table[row sep=crcr]{%
3	0.295246\\
5	0.2633\\
10	0.28871\\
20	0.290583\\
};
\addlegendentry{nojd1}

\node[below, align=center]
at (rel axis cs:0.5,1) {${\text{\it{} n = 100}}$};
\end{axis}
\end{tikzpicture}%